\documentclass{article}

\PassOptionsToPackage{round}{natbib}
\usepackage{colm2026}
\usepackage[utf8]{inputenc}
\usepackage[T1]{fontenc}
\usepackage{microtype}

\usepackage{amsmath}
\usepackage{amssymb}
\usepackage{amsfonts}
\usepackage{mathtools}
\usepackage{amsthm}
\usepackage{nicefrac}
\usepackage{physics}
\usepackage{bbm}

\theoremstyle{plain}
\newtheorem{theorem}{Theorem}[section]
\newtheorem{proposition}{Proposition}[section]
\newtheorem{lemma}{Lemma}[section]
\newtheorem{corollary}{Corollary}[section]
\theoremstyle{definition}
\newtheorem{definition}{Definition}[section]
\newtheorem{assumption}{Assumption}[section]
\theoremstyle{remark}
\newtheorem{remark}{Remark}[section]

\usepackage{natbib}

\usepackage{hyperref}
\usepackage[all]{hypcap}
\usepackage{xcolor}
\definecolor{teal}{HTML}{008080}
\hypersetup{
  colorlinks,
  citecolor=teal,
  linkcolor=red!80!black,
  urlcolor=blue
}

\usepackage[nameinlink,capitalize,noabbrev]{cleveref}
\crefname{appendix}{Appendix}{Appendices}
\Crefname{appendix}{Appendix}{Appendices}
\crefname{theorem}{Theorem}{theorems}
\Crefname{theorem}{Theorem}{Theorems}
\crefname{proposition}{Proposition}{propositions}
\Crefname{proposition}{Proposition}{Propositions}
\crefname{lemma}{Lemma}{lemmas}
\Crefname{lemma}{Lemma}{Lemmas}
\crefname{corollary}{Corollary}{corollaries}
\Crefname{corollary}{Corollary}{Corollaries}
\crefname{definition}{Definition}{definitions}
\Crefname{definition}{Definition}{Definitions}
\crefname{assumption}{Assumption}{assumptions}
\Crefname{assumption}{Assumption}{Assumptions}
\crefname{remark}{Remark}{remarks}
\Crefname{remark}{Remark}{Remarks}

\makeatletter
\g@addto@macro\appendix{\crefalias{section}{appendix}}
\makeatother

\usepackage{graphicx}
\usepackage{subcaption}
\usepackage{wrapfig}
\usepackage{float}
\usepackage{pdfpages}
\usepackage{booktabs}

\usepackage{verbatim}
\usepackage{comment}
\usepackage{xspace}

\usepackage[acronym,nowarn]{glossaries}
\glsdisablehyper
\makeglossaries

\usepackage{tikz,siunitx}
\usetikzlibrary{shapes.geometric,shapes.symbols}

\usepackage[most]{tcolorbox}
\definecolor{mylightgray}{gray}{0.95}
\definecolor{boxcolor}{HTML}{faf9f5}
\definecolor{takeawaybg}{HTML}{DCEBFF}
\definecolor{takeawayframe}{HTML}{5E90DA}
\definecolor{takeawaytitle}{HTML}{234C86}
\newcommand{\takeawayextratitle}{}
\newtcolorbox{mybox}{colback=mylightgray,colframe=mylightgray,top=0.8pt,bottom=0.8pt,right=1.8pt,left=1.8pt}
\newtcolorbox[auto counter]{takeaway}[1][]{%
  enhanced,
  breakable,
  colback=takeawaybg,       %
  colframe=takeawayframe,   %
  boxrule=1.25pt,
  arc=3pt,                  %
  left=2mm,right=2mm,top=1mm,bottom=1mm,
  before skip=10pt, after skip=10pt,
  takeaway title/.store in=\takeawayextratitle,
  takeaway title=,
  title={Takeaway~\thetcbcounter\ifx\takeawayextratitle\empty\else:~\takeawayextratitle\fi},
  colbacktitle=takeawaytitle,%
  coltitle=white,           %
  fonttitle=\bfseries\small,
  attach boxed title to top left={yshift=-1.2mm, xshift=2mm},
  boxed title style={
    enhanced,
    arc=3pt,
    top=0.5mm, bottom=0.5mm, left=1mm, right=1mm,
    boxrule=0pt,           %
    interior engine=empty, %
  },
  #1                        %
}

\usepackage[textsize=tiny]{todonotes}

\newacronym{AU}{AU}{Aleatoric Uncertainty}
\newacronym{AUC}{AUC}{area under the curve}
\newacronym{BDL}{BDL}{Bayesian Deep Learning}
\newacronym[firstplural=Bayesian neural networks]{BNN}{BNN}{Bayesian neural network}
\newacronym{CNN}{CNN}{convolutional neural network}
\newacronym{DE}{DE}{Deep Ensembles}
\newacronym[firstplural=Deep Gaussian Processes]{DGP}{DGP}{Deep Gaussian Process}
\newacronym[firstplural=deep neural networks]{DNN}{DNN}{deep neural network}
\newacronym{ECE}{ECE}{expected calibration error}
\newacronym{ELBO}{ELBO}{evidence lower bound}
\newacronym{ELL}{ELL}{expected log likelihood}
\newacronym{ESS}{ESS}{Effective Sample Size}
\newacronym{EU}{EU}{Epistemic Uncertainty}
\newacronym{GLM}{GLM}{Generalized linear model}
\newacronym{GP}{GP}{Gaussian process}
\newacronym{GGN}{GGN}{Generalized Gauss-Newton}
\newacronym{HMC}{HMC}{Hamiltonian Monte Carlo}
\newacronym{ICL}{ICL}{In-Context Learning}
\newacronym{IVON}{IVON}{Improved Variational Online Newton}
\newacronym{KL}{KL}{Kullback-Leibler divergence}
\newacronym{LLA}{LLA}{Linearized Laplace Approximation}
\newacronym{LLC}{LLC}{Local Learning Coefficient}
\newacronym{LL}{LL}{log-likelihood}
\newacronym{LLM}{LLM}{Large Language Model}
\newacronym{MAP}{MAP}{Maximum a posteriori}
\newacronym{MC}{MC}{Monte Carlo}
\newacronym{MCD}{MCD}{Monte Carlo Dropout}
\newacronym{MCMC}{MCMC}{Markov chain Monte Carlo}
\newacronym{MI}{MI}{Mutual Information}
\newacronym{MLP}{MLP}{multilayer perceptron}
\newacronym{MNLL}{MNLL}{mean negative loglikelihood}
\newacronym{NLL}{NLL}{negative loglikelihood}
\newacronym{NMLL}{NMLL}{Negative Marginal Log Likelihood}
\newacronym{NN}{NN}{neural network}
\newacronym{NTK}{NTK}{Neural Tangent Kernel}
\newacronym{OOD}{OOD}{out-of-distribution}
\newacronym{PSD}{PSD}{Positive Semi-Definite}
\newacronym{RLCT}{RLCT}{Real Log Canonical Threshold}
\newacronym{RELU}{ReLU}{rectified linear unit}
\newacronym{RMSE}{RMSE}{root mean square error}
\newacronym{SAM}{SAM}{Sharpness Aware Minimization}
\newacronym{SGLD}{SGLD}{Stochastic Gradient Langevin Dynamics}
\newacronym{SGHMC}{SGHMC}{Stochastic Gradient Hamiltonian Monte Carlo}
\newacronym{SGD}{SGD}{Stochastic Gradient Descent}
\newacronym{SLT}{SLT}{Singular Learning Theory}
\newacronym{TU}{TU}{Total Uncertainty}
\newacronym{UQ}{UQ}{Uncertainty Quantification}
\newacronym{VI}{VI}{Variational Inference}
\newacronym{WD}{WD}{Wasserstein distance}

\usepackage{math_commands}

\title{
  A Bayesian Perspective on the Role of Epistemic Uncertainty for Delayed Generalization in In-Context Learning
}

\author{%
  Abdessamed Qchohi \\ abdessamed.qchohi@eurecom.fr \\
  EURECOM, France
  \And
  Simone Rossi \\ simone.rossi@eurecom.fr \\
  EURECOM, France
}

\begin{document}
\maketitle

\begin{abstract}
  In-context learning enables transformers to adapt to new tasks from a few examples at inference time, while grokking highlights that this generalization can emerge abruptly only after prolonged training.
  We study task generalization and grokking in in-context learning using a Bayesian perspective, asking what enables the delayed transition from memorization to generalization.
  Concretely, we consider modular arithmetic tasks in which a transformer must infer a latent linear function solely from in-context examples and analyze how predictive uncertainty evolves during training.
  We combine approximate Bayesian techniques to estimate the posterior distribution and we study how uncertainty behaves across training and under changes in task diversity, context length, and context noise.
  We find that epistemic uncertainty collapses sharply when the model groks, making uncertainty a practical label-free diagnostic of generalization in transformers.
  Additionally, we provide theoretical support with a simplified Bayesian linear model, showing that asymptotically both delayed generalization and uncertainty peaks arise from the same underlying spectral mechanism, which links grokking time to uncertainty dynamics.
\end{abstract}

\section{Introduction}
\label{sec:introduction}

\Gls{ICL} has emerged as a powerful paradigm in which \glspl{LLM} can learn to perform new tasks by conditioning on a few examples provided in the input prompt \citep{brown2020language}.
However, despite its empirical success, the underlying mechanisms that enable LLMs to effectively learn from in-context examples remain poorly understood. Recent studies have begun to explore the theoretical foundations of ICL, proposing various hypotheses ranging from implicit Bayesian inference \citep{xie2022explanation,akyurek2023what} to meta-learning frameworks \citep{von_oswald2023transformers}. Yet, a comprehensive Bayesian perspective on how LLMs process and learn from in-context examples is still lacking.
Furthermore, during training of \glspl{LLM}, models often exhibit a phenomenon known as \emph{grokking}, where they suddenly transition from poor to near-perfect performance on a task after extended training, despite having already fit the training data \citep{power2022grokking}.
In this work, we aim to study what enables grokking in \gls{ICL} from a Bayesian perspective and how learning dynamics (and uncertainty measures) evolve during training.
\gls{BDL} provides a principled framework for modeling uncertainty in deep learning models by treating model parameters as random variables and inferring their posterior distributions given the data \citep{gal2016dropout,blundell2015weight,lakshminarayanan2017simple}.
While exact inference is intractable in large-scale models like \glspl{LLM}, various approximate inference techniques, such as \gls{VI} and Laplace approximations, have been successfully applied to capture uncertainty in deep learning \citep{papamarkou2024position}.
By treating the model in a Bayesian framework, we experimentally analyze how uncertainty quantification and posterior distributions over model parameters change as the model groks with the in-context examples.
Following the classic setting, we approch this phenomenon by analyzing the training dynamics of transformers on modular arithmetic tasks \citep{power2022grokking}.
Additionally, we provide a theoretical support for our empirical findings by analyzing the uncertainty dynamics in a simplified model, showing how the transition to generalization corresponds to a collapse in epistemic uncertainty.
\begin{wrapfigure}[18]{r}{0.5\textwidth}
  \vspace{-0.5cm}
  \centering
  \includegraphics[width=0.5\textwidth]{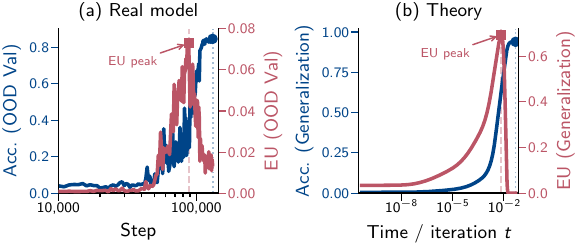}
  \vspace{-0.5cm}
  \caption{
    \textbf{Epistemic uncertainty collapses sharply at the grokking transition.}
    On a decoder-only transformer trained on modular arithmetic tasks, the epistemic uncertainty collapses sharply before the generalization transition.
    This pattern follows the theoretical prediction for a simplified linear model, where the grokking transition corresponds to a spectral shift in the posterior distribution that causes a collapse in epistemic uncertainty.
  }
  \label{fig:teaser}
\end{wrapfigure}
Our contributions are as follows:
(1) We study grokking in \gls{ICL} in a Bayesian framework, using modular arithmetic tasks to analyze how predictive uncertainty evolves as a transformer learns to infer a latent task from in-context examples.
(2) We show empirically that a sharp collapse in epistemic uncertainty coincides with the grokking transition, making it a label-free diagnostic of generalization. We further compare \gls{IVON} and Last-Layer Laplace across training, task diversity, context length, and context noise, highlighting both shared uncertainty patterns and method-specific differences in \gls{OOD} generalization.
(3) We provide theoretical support with a simplified Bayesian linear model, showing that delayed generalization and delayed epistemic-uncertainty peaks arise from the same underlying spectral mechanism, which links grokking time to uncertainty dynamics.
A snapshot of our main empirical and theoretical findings is shown in \cref{fig:teaser}.

\paragraph{Related Works.}

\gls{ICL} has been studied both empirically and theoretically since its emergence in large language models \citep{brown2020language}.
A common theoretical view is that transformers can implement learning or estimation procedures in their forward pass, including implicit Bayesian inference, ridge regression, and gradient-based adaptation \citep{xie2022explanation,akyurek2023what,von_oswald2023transformers,bai2023transformers,xie2025initialization,zhang2025what}.
Complementary work examines reliability under prompting: \citet{zhang2024study} study calibration across shot counts, \citet{ling2024uncertainty} decompose \gls{ICL} uncertainty into epistemic and aleatoric terms, and \citet{wang2025uncertainty} analyze uncertainty in long-context prompting.
A Bayesian framing provides a natural language for connecting these threads: pre-training can be interpreted as learning an implicit prior (or an implicit family of priors) over tasks/concepts, while ICL corresponds to approximate Bayesian inference (or Bayesian model averaging) over latent task variables given the prompt \citep{xie2022explanation, zhang2025what}.
In parallel, \gls{BDL} treats uncertainty as first-class, via priors and posteriors over parameters and predictive distributions as marginalization, and offers tools that make uncertainty trajectories during training and inference measurable \citep[e.g.][]{graves2011practical,blundell2015weight,gal2016dropout,lakshminarayanan2017simple, maddox2019simple,ritter2018scalable}.
On the grokking side, \citet{power2022grokking} introduced the canonical delayed-generalization phenomenon on modular arithmetic. Follow-up work characterized its learning phases and representation dynamics \citep{liu2022towards}, explained aspects of modular addition through the transition out of the \gls{NTK} regime \citep{mohamadi2024why}, and proposed information-theoretic or structural progress measures that can anticipate grokking \citep{clauw2024information,golechha2024progress}.
Closest to our setup, \citet{he2024learning} extend the classic task family to modular linear functions and study both in-distribution and out-of-distribution learning.
While we build on the same experimental set up, our focus is on the Bayesian learning dynamics and uncertainty trajectories during grokking, rather than characterizing the phases of learning or the representational changes.
Notably, we identify a the condition under which the model can generalize in-domain and out-of-domain, and we analyze how the model's uncertainty evolves as it transitions from overfitting to generalization, providing a novel perspective on the grokking phenomenon.
A more detailed review of the related work is available in \cref{sec:related-works-long}.

\section{Preliminaries}
\label{sec:background}

\paragraph{Bayesian inference.}
Rather than treating model parameters as fixed values, Bayesian methods model them as random variables.
Given a dataset $\mathcal{D}$ and model parameters $\mbtheta$, Bayes' rule gives the posterior:
$$
p(\mbtheta \mid \mathcal{D}) = \frac{p(\mathcal{D} \mid \mbtheta)\, p(\mbtheta)}{p(\mathcal{D})}
$$
where $p(\theta)$ is the prior encoding beliefs about the parameters before observing data, $p(\mathcal{D} \mid \mbtheta)$ is the likelihood describing how probable the observed data is under a particular choice of parameters,
and $p(\mathcal{D}) = \int p(\mathcal{D} \mid \mbtheta)\, p(\mbtheta)\, \dd\mbtheta$ is the marginal likelihood, which integrates out the parameters and acts as a normalizing constant ensuring the posterior is a valid probability
distribution.
Predictions for a new input $\mbx^*$ are obtained by
marginalizing over the posterior:
\begin{equation}
  \label{eq:marginalization}
  p(\mby^* \mid \mbx^*, \mathcal{D}) = \int p(\mby^* \mid \mbx^*, \mbtheta)\,
  p(\mbtheta \mid \mathcal{D})\, d\mbtheta
\end{equation}
This predictive distribution captures uncertainty in the model parameters, which is relevant when data is limited or when the model is evaluated on out-of-distribution inputs.
Exact inference is intractable for neural networks since the posterior is high-dimensional and the marginal likelihood
has no closed form.
We therefore need to rely on approximate inference methods.

While the toolbox of approximate inference methods is large and rapidly evolving, we focus on two complementary approaches that are well-suited for our setting: variational inference and Laplace approximation.
\gls{IVON} \citep{shen2024variational} integrates uncertainty estimation into training by maintaining a Gaussian variational posterior over all weights, updated via Newton-like steps at a computational cost comparable to Adam.
The Last-Layer Laplace Approximation \citep{daxberger2021laplace} is a post-hoc method that fits a Gaussian approximation to the posterior of the final linear layer after standard MAP training, using a Kronecker-factored Hessian approximation, leaving all other parameters fixed.
For both methods, predictions are obtained by averaging over $S$ weight samples drawn from $q(\mbtheta)$.
Additional details on these methods are available in \cref{sec:methods}.

\paragraph{Uncertainty decomposition.}
Both approximate inference methods yield a distribution over predictions, enabling a principled decomposition
of predictive uncertainty \citep{wimmer2023quantifying,rosso2025scaling,h_ullermeier2021aleatoric}.
Given $S$ weight samples $\mbtheta^{(s)}$ from the approximate posterior, total uncertainty (TU) is the
entropy of the mean prediction; aleatoric uncertainty (AU) is the mean entropy of individual predictions and captures irreducible label noise; and epistemic uncertainty (EU) is their difference:
\[
  \mathrm{EU}(\mbx) =
  \mathbb{H}\!\left[\frac{1}{S}\sum_{s=1}^{S} p(\mby \mid \mbx, \mbtheta^{(s)})\right] -
  \underbrace{\frac{1}{S}\sum_{s=1}^{S} \mathbb{H}\!\left[p(\mby \mid \mbx, \mbtheta^{(s)})\right]}_{\mathrm{AU}(\mbx)}
\]
Epistemic uncertainty reflects the model's lack of knowledge about the task
and is expected to decrease as the model acquires more information. In the
grokking setting, we use EU as a probe for the generalization transition.

\section{In-Context Learning, Grokking and Uncertainty}
\label{sec:methodology}

In this section, we describe the main experimental results on the relationship between grokking and uncertainty in in-context learning.
We first introduce the experimental setup, including the task distribution, model architecture, and evaluation protocol.
We then present our main empirical findings on how accuracy and epistemic uncertainty behave.
Finally, we analyze how task diversity, context length, and context noise affect both generalization and uncertainty dynamics.

\subsection{Experimental Setup}

\paragraph{Tasks.} We follow \citet{he2024learning} and focus on linear modular
arithmetic tasks of the form $z = ax + by \pmod{p}$, where $(a, b)$ is the
task vector and $(x, y)$ is the input vector, with all values in
$\{0, \ldots, p-1\}$. Modular arithmetic provides a controlled setting for
studying grokking, as the tasks have a clear underlying structure that the
model can either memorize or learn to generalize. The task vector is never
shown to the model, which receives only input-output triples $(x, y, z)$
and must infer the underlying task from the in-context examples. Each
equation is tokenized into three numerical. An example of an input sequence is:
\[
  (x_1, y_1, z_1),\; \ldots,\; (x_{n_\mathrm{ctx}}, y_{n_\mathrm{ctx}},
  z_{n_\mathrm{ctx}}),\; (x_t, y_t, \textbf{?})
\]
During training, the
model predicts every third token $z_i$ as well as the final target $z_t$,
with the cross-entropy loss computed only on these positions. At evaluation time, we measure accuracy on $z_t$ only, which
directly tests whether the model has inferred the correct task from context.
We partition the set of task vectors and input examples into seen and unseen subsets, $\mathcal{T}_\mathrm{id}$, $\mathcal{T}_\mathrm{ood}$ and $\mathcal{X}_\mathrm{train}$, $\mathcal{X}_\mathrm{test}$, respectively.
Crossing these two partitions gives four splits: ID Train ($\mathcal{T}_\mathrm{id}$, $\mathcal{X}_\mathrm{train}$), used for training; ID Val ($\mathcal{T}_\mathrm{id}$, $\mathcal{X}_\mathrm{test}$), which tests generalization to unseen inputs on seen tasks; OOD Train ($\mathcal{T}_\mathrm{ood}$, $\mathcal{X}_\mathrm{train}$), which tests whether the model adapts to unseen tasks on seen inputs; and OOD Val ($\mathcal{T}_\mathrm{ood}$, $\mathcal{X}_\mathrm{test}$), the hardest split, which tests full out-of-distribution generalization to both unseen tasks and unseen inputs.
We train exclusively on ID Train and use the remaining three splits to evaluate generalization.
This construction lets us track whether the model memorizes, generalizes in-distribution, or generalizes out-of-distribution, and at which point grokking occurs in each regime.

\paragraph{Model architecture and training.}
We use a decoder-only Transformer with Rotary Positional Embeddings (RoPE) \citep{su2024roformer} with $d=6$ layers,
embedding dimension $512$, $4$ attention heads, and an MLP widening factor of $4$.
We tie input embedding and output projection weights and apply weight decay for regularization.
Task vectors are sampled via a parallelogram rule, varying one component at a time so the model can build on previously
learned structure, and each batch is constructed so that all tasks in $\mathcal{T}_\mathrm{id}$ share the same input sequences, limiting inter-task interference.
Unless otherwise specified, we train four random seeds and we report medians with $95\%$ quantile intervals across seeds for all metrics and configurations.
Detailed descriptions of the training procedure, hyperparameters, and approximate inference methods are provided in \cref{sec:reproducibility}.

\section{Empirical Results and Analysis}

\begin{figure}[t]
  \centering
  \textbf{IVON}\par\medskip
  \includegraphics[width=0.99\textwidth]{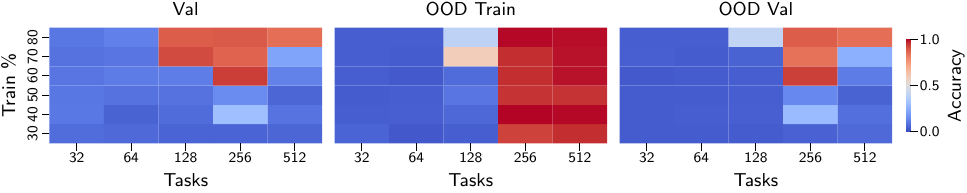}
  \textbf{Laplace}\par\medskip
  \includegraphics[width=0.99\textwidth]{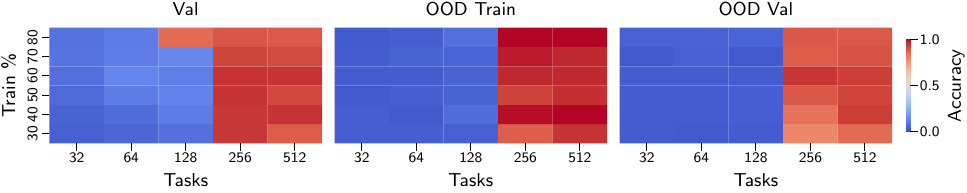}
  \caption{
    \textbf{Accuracy heatmap with IVON and Laplace}.
    For OOD Train, both IVON and Laplace achieve high accuracy starting at $T = 256$ and $T = 512$ for any training data fraction. For OOD Val (the most challenging split), Laplace generalizes more broadly across configurations, while IVON achieves competitive performance only at high training data fractions.
  }
  \label{fig:pretrain-final-acc}
\end{figure}
\paragraph{Grokking across task diversity.}
To study how grokking emerges as a function of task diversity and training data, we train the model across a grid of configurations varying the number of pre-training tasks $T \in \{32, 64, 128, 256, 512\}$ and the fraction
of training data $\alpha \in \{0.3, 0.4, 0.5, 0.6, 0.7, 0.8\}$, running four random seeds per configuration.
We measure accuracy on Val, OOD Train, and OOD Val, and compare IVON and Laplace to
assess whether the choice of optimizer affects the generalization landscape.
\cref{fig:pretrain-final-acc} shows the results.
Both IVON and Laplace achieve high OOD Train accuracy at $T = 256$ and $T = 512$.
For OOD Val, Laplace exhibits stronger and more consistent generalization across task counts and data fractions, whereas IVON reaches competitive OOD Val accuracy only at the highest training data fractions.
OOD Val generalization emerges only at the highest task counts for both methods and requires careful tuning of training data fraction.

\paragraph{Epistemic uncertainty tracks grokking.}
To test whether epistemic uncertainty tracks the grokking transition, we
use IVON and Laplace to decompose predictive uncertainty into epistemic and aleatoric
components across the same hyperparameter grid.
\cref{fig:pretrain-uncertainty-decomp} shows the results in each evaluation set.
\begin{figure}[t]
  \centering
  \textbf{IVON}\par\medskip
  \includegraphics[width=0.99\textwidth]{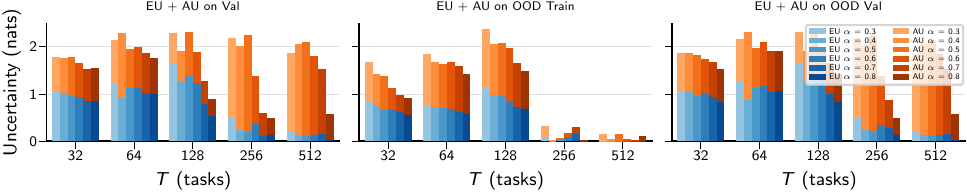}
  \textbf{Laplace}\par\medskip
  \includegraphics[width=0.99\textwidth]{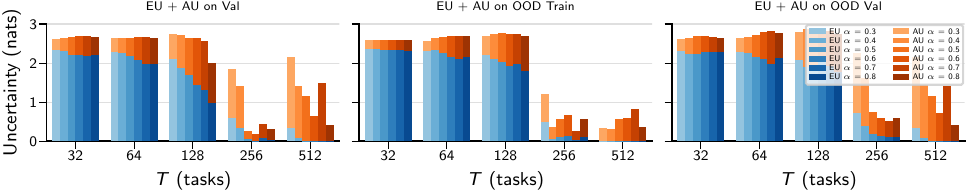}

  \caption{
    \textbf{Uncertainty decomposition with \gls{IVON} and Laplace}.
    Decomposition of predictive uncertainty into aleatoric and epistemic
    components across the pretraining hyperparameter grid.
  }
  \label{fig:pretrain-uncertainty-decomp}
\end{figure}
Both methods show that at low task diversity, total uncertainty is dominated
by the epistemic component across all splits. As $T$ increases, epistemic
uncertainty drops sharply at $T = 256$ and $T = 512$ for both IVON and
Laplace, while aleatoric uncertainty remains high. This reduction closely
tracks the generalization transition in the accuracy heatmaps, providing a
principled uncertainty-based signature of grokking consistent across approximate
inference methods.
This is an important finding: as of this writing, \emph{grokking} has been primarily characterized in terms of accuracy dynamics (which implies the need for labels to detect it), while here we show that it also has a clear signature in the model's posterior uncertainty, which can be measured without access to labels and provides a more direct window into the model's learning dynamics.

\paragraph{Effect of context length on generalization.}
To study how the number of in-context examples affects the model's ability
to infer the underlying task, we take the models trained with $n_\mathrm{ctx}
= 32$ and evaluate them by varying the number of in-context examples from
$1$ to $32$ at inference time, without any additional training. The task
vector $(a, b)$ is never provided to the model, which must infer it solely
from the input-output triples in the context. We compare IVON and Laplace
on Val, OOD Train, and OOD Val. \cref{fig:context-sweep-eu-au} shows
accuracy, epistemic uncertainty, and aleatoric uncertainty as a function
of context size.
\begin{figure}[t]
  \centering
  \includegraphics[width=0.99\textwidth]{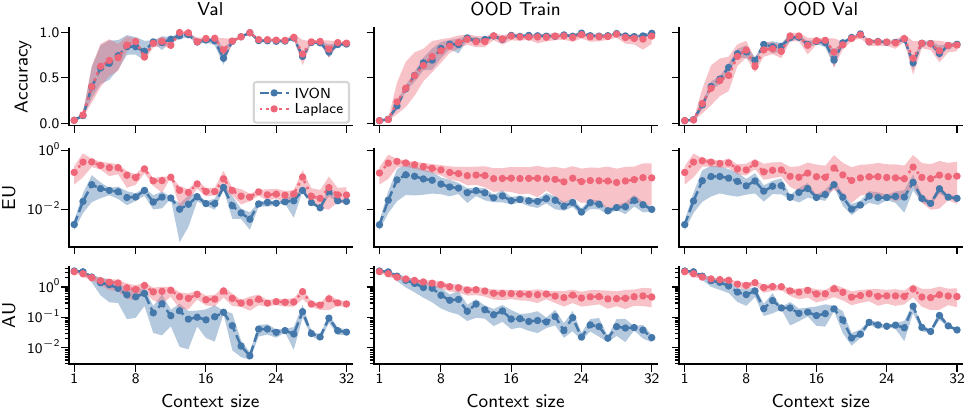}
  \caption{
    \textbf{Accuracy, epistemic and aleatoric uncertainty across context sizes}.
    Increasing the number of in-context examples monotonically improves accuracy and decreases alepatoric uncertainty, while the epistemic uncertainty follows a non-monotonic behavior.
  }
  \label{fig:context-sweep-eu-au}
\end{figure}
Accuracy improves monotonically with context size for both IVON and Laplace across all splits, confirming that more examples indeed help the model infer the task (and thus generalize better on OOD splits).
As expected, the aleatoric uncertainty decreases with more context and converges to a floor, reflecting the reduced ambiguity in the task given more evidence.
The behavior of epistemic uncertainty is more complex: it increases at small context sizes, likely due to the model's initial uncertainty about the task, then decreases sharply as more examples are added.
Notably, this behavior is present in both IVON and Laplace, albeit with different magnitudes, suggesting that it isn't an artifact of the approximate inference method but rather a property of the model's (meta-)learning dynamics as it processes in-context evidence.
Also, note that this evidence does not contradict the known fact that uncertainty scales monotonically with data in Bayesian inference \citep[see, e.g.,][]{rosso2025scaling}, but it rather highlights that the current scaling laws of uncertainty should be refined to account for the effect of test-time adaptation via in-context learning.

\begin{figure}[t]
  \centering
  \includegraphics[width=0.95\textwidth]{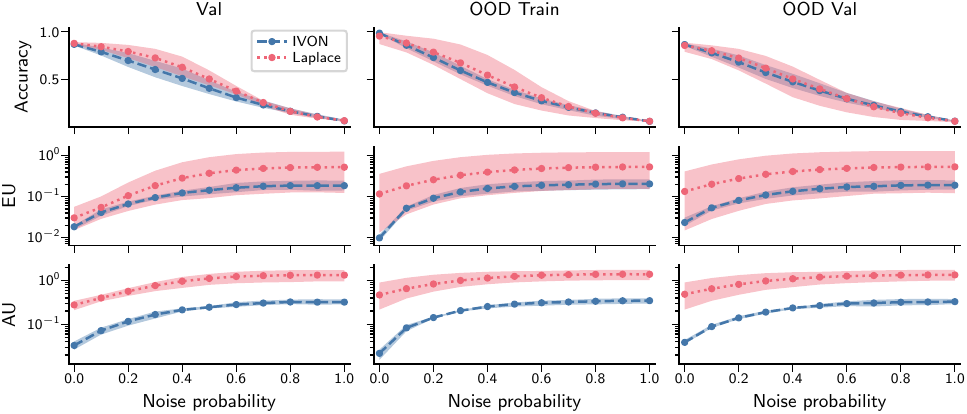}
  \caption{
    \textbf{Accuracy vs. epistemic and aleatoric uncertainty across context
    noise levels}. Corrupting in-context examples with random label noise degrades accuracy and raises both aleatoric and epistemic uncertainty.
  }
  \label{fig:context-noise-sweep-eu-au}
\end{figure}
\paragraph{Robustness to context noise.}
As a final test of the model's uncertainty estimates, we evaluate how they respond to noise in the in-context examples.
For this experiment, we corrupt the context by replacing each label $z_i$ in the context triples with a random token sampled uniformly from $\{0, \ldots, p-1\}$ with probability $p_\mathrm{noise} \in [0, 1]$.
The input tokens $x$ and $y$ are never corrupted, and the final query triple is always clean.
As we did before, we evaluate both IVON and Laplace on Val, OOD Train, and OOD Val, measuring accuracy, aleatoric uncertainty, and epistemic uncertainty as a function of the noise level (\cref{fig:context-noise-sweep-eu-au}).
Accuracy degrades smoothly as noise increases for both methods and all
splits. Aleatoric uncertainty rises proportionally to the noise level,
correctly capturing the increased label ambiguity in the context. Epistemic
uncertainty also increases with noise, but more markedly for Laplace than
for IVON, suggesting that the variational posterior learned by IVON is more
robust to corrupted context.
This is slightly surprising, as one might expect that corrupted context would primarily increase epistemic uncertainty, since it makes the task harder to infer, while aleatoric uncertainty should remain stable as it reflects irreducible noise in the data rather than lack of context.
This demonstrates that recent analysis on disentangling aleatoric and epistemic uncertainty \citep[e.g.,][]{mucs_anyi2024benchmarking,kirchhof2025position} do not necessarily apply to the in-context learning setting, and that further work is needed to understand how these uncertainty components interact when the model is performing test-time adaptation.

\section{Why Epistemic Uncertainty Can Also Define Grokking Time}
\label{sec:theoretical-analysis}

Our empirical results suggest that grokking is visible not only in accuracy curves but also in epistemic uncertainty. Making this claim rigorous for the full transformer on modular arithmetic is currently out of reach, so in this section we drastically simplify the problem.
Following the linear-regression setup of \citet{levi2024grokking}, we replace the sequence model by Bayesian linear regression.
Although this is a very simple model, \citet{levi2024grokking} show that it already captures the key phenomenon of delayed generalization.
Here we show that it also captures the key phenomenon of delayed EU peaks, when the model is trained with Gaussian variational inference.
This abstraction is clearly much simpler than the model studied in the rest of the paper, but it gives fully tractable dynamics and still matches the qualitative empirical picture: training error decays before generalization
error, the epistemic uncertainty is non-monotone, and the generalization-side transition is delayed.
While full derivations are deferred to \cref{sec:simplified-model}, we sketch the main results here.

\paragraph{Simplified problem.}
Let $\mbX \in \bbR^{n \times d}$ be the design matrix, with $\mbx_i\sim \cN(\mbzero, \mbI_d)$, and $\mby \in \bbR^n$ be the target vector, generated from $\mbw_\star$ with isotropic Gaussian prior.
Consider the linear-Gaussian model
\begin{align}
  p(\mby \mid \mbw)
  & = \cN(\mby; \mbX \mbw, \sigma^2 \mbI_n),
  &
  p(\mbw)
  & = \cN(\mbw; \mbzero, \tau^2 \mbI_d).
\end{align}
We study the regime $n, d \to \infty$ with fixed aspect ratio
$\lambda = d / n < 1$.
Our goal is to analyze the dynamics of Gaussian variational inference in this model, and to understand how the risk and uncertainty evolve along the variational trajectory.
Following the results of \citet{lambert2022variational}, since the model is linear and the variational family is Gaussian,
the gradient flows on the Bures--Wasserstein space of Gaussian measures $q_t(\mbw) = \cN(\mbw; \mbmu_t, \mbSigma_t)$ can be written in closed form as a system of ODEs:
\begin{align}
  \dot \mbmu_t
  & = \mbb - \mbA \mbmu_t,
  \\
  \dot \mbSigma_t
  & = 2 \mbI_d - \mbA \mbSigma_t - \mbSigma_t \mbA.
\end{align}
with $\mbA = \sigma^{-2}\mbX^\top \mbX + \tau^{-2}\mbI_d$ and $\mbb = \sigma^{-2}\mbX^\top \mby$.
This system has exact solutions, as shown in the Appendix.
Along the Gaussian variational trajectory, we are interested in the expected
risks
\begin{align}
  \mathcal{L}_{\text{emp}}(t) = \E_{q_t}\frac{1}{n}\sum_{i=1}^n \ell(y_i, \mbx_i^\top \mbw)
  \qquad \text{and} \qquad
  \mathcal{L}_{\text{gen}}(t) = \E_{q_t(\mbw)p(\mbx, y)}\ell(y_i, \mbx_i^\top \mbw) %
\end{align}
where $\ell$ is the squared loss, i.e. $\ell(y, \hat y) = \frac{1}{2}(y - \hat y)^2$.
We can also define the excess risks, which are the differences between the risks at time $t$ and their limiting values as $t \to \infty$.
\paragraph{Risk dynamics and spectral decomposition.}
In the well-specified model regime, the fixed point $\mbmu_\infty$ coincides
with the exact parameters $\mbw_\star$ and the population covariance is isotropic, so the relevant
objects are the mean gap $\delta\mbmu_0 \equiv \mbmu_0 - \mbmu_\infty$ and the
covariance gap $\Delta\mbSigma_0 \equiv \mbSigma_0 - \mbSigma_\infty$. Writing
$\mbE_t \equiv e^{-\mbA t}$, the excess empirical and generalization risks are
\begin{align}
  \widetilde{\mathcal{L}}_{\mathrm{emp}}(t)
  & =
  \frac{1}{2n}
  \delta\mbmu_0^\top \mbE_t \mbX^\top \mbX \mbE_t \delta\mbmu_0
  + \frac{1}{2n}\Tr(\mbX^\top \mbX \mbE_t \Delta\mbSigma_0 \mbE_t),
  \\
  \widetilde{\mathcal{L}}_{\mathrm{gen}}(t)
  & =
  \frac{1}{2}\delta\mbmu_0^\top \mbE_t^2 \delta\mbmu_0
  + \frac{1}{2}\Tr(\mbE_t \Delta\mbSigma_0 \mbE_t).
\end{align}
The first term tracks the residual bias of the variational mean, while the second tracks the residual uncertainty from the covariance.
Crucially, both risks are driven by the same operator $\mbA$, but they are read out in two different quantities: the empirical risk uses the training covariance $\widehat{\mbSigma} = \mbX^\top \mbX / n$, whereas the population risk uses the identity.
In the high-dimensional limit, the eigenvalues of the sample covariance $\widehat{\mbSigma}$ converge in distribution to the Marchenko-Pastur law $\mathrm{MP}(\lambda)$, and the excess risks can be expressed as expectations over this distribution:
\begin{align}
  \label{eq:vi_emp_loss_anal}
  \widetilde{\mathcal{L}}_{\text{emp}}(t)
  & \approx
  C_0\,
  \E_{\nu \sim \mathrm{MP}(\lambda)}
  \left[
    \nu \exp\left(
      -2 \left(
        \frac{n}{\sigma^2}\nu + \frac{1}{\tau^2}
      \right)t
    \right)
  \right],
  \\
  \label{eq:vi_gen_loss_anal}
  \widetilde{\mathcal{L}}_{\text{gen}}(t)
  & \approx
  C_0\,
  \E_{\nu \sim \mathrm{MP}(\lambda)}
  \left[
    \exp\left(
      -2 \left(
        \frac{n}{\sigma^2}\nu + \frac{1}{\tau^2}
      \right)t
    \right)
  \right].
\end{align}
where $C_0$ collects the initial mean and covariance gaps (details in the Appendix).
Under the assumption of high-dimensional self-averaging regime, the late-time behavior is controlled by the lower Marchenko-Pastur edge $\nu_- = (1 - \sqrt{\lambda})^2$:
\begin{align}
  \widetilde{\mathcal{L}}_{\text{emp,gen}}(t)
  \sim
  K_{\text{emp,gen}} t^{-3/2} e^{-\kappa t},
  \qquad
  \kappa
  =
  2\left(
    \frac{n}{\sigma^2}(1 - \sqrt{\lambda})^2 + \frac{1}{\tau^2}
  \right),
\end{align}
with ratio $K_{\text{gen}} / K_{\text{emp}} = (1 - \sqrt{\lambda})^{-2}$.
The common exponential rate $\kappa$ reflects the shared edge mode, while the
different factors is due to the extra $\nu$ in the empirical risk and it is responsible for the train--test delay.
This asymptotic mismatch is the only ingredient needed for the uncertainty result below.

\begin{figure}[t]
  \centering
  \includegraphics[width=\textwidth]{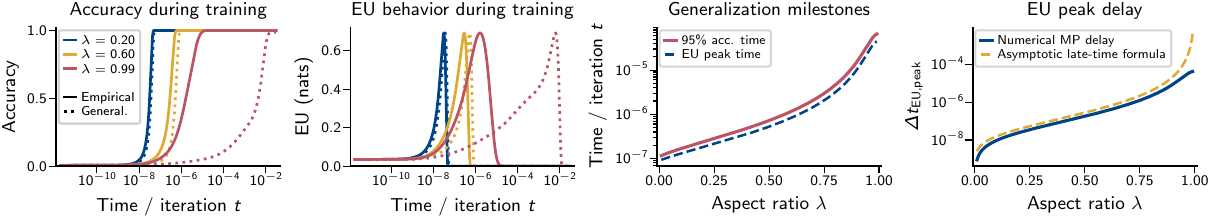}
  \caption{
    \textbf{Training and generalization dynamics of the simplified model.}
    Increasing the aspect ratio $\lambda$ delays generalization relative to training (\textit{left}).
    The epistemic uncertainty peaks coincide with the transition from interpolation to generalization (\textit{center left}), although the EU peak occurs before the accuracy threshold is exceeded (\textit{center right}).
    Consistently, the EU peak delay grows rapidly with $\lambda\to 1$, with the asymptotic formula of \cref{prop:threshold-eu-peak-delay} providing a good proxy for the observed delay (\textit{right}).
  }
  \label{fig:theory-summary}
\end{figure}

\paragraph{Threshold-model uncertainty and delayed EU peaks.}
Following the construction used by \citet{levi2024grokking}, we turn regression into a binary success event by choosing a tolerance $\epsilon > 0$ and declaring a prediction correct whenever the expected squared residual is smaller than $\epsilon$.
For a datapoint $(\mbx, y)$, define the thresholded success variable
$
Z_\epsilon(\mbx, y; \mbw)
=
\mathbbm{1}\{(y - \mbx^\top \mbw)^2 \le \epsilon\}
$.
We can define the empirical and generalization accuracies $\mathcal A_{\text{emp}}$ and $\mathcal A_{\text{gen}}$
as the expected value of $Z_\epsilon$ under the variational distribution and the data distribution (see the Appendix for details).
By analogy with the accuracy threshold defined in \citet{levi2024grokking}, we can define the grokking time as the time at which the accuracy curve crosses a certain threshold, which is equivalent as the time at which the excess risk crosses a certain threshold; for $95\%$ accuracy we have $\widetilde{\mathcal{L}}(t_\mathrm{grok}) \approx \epsilon / 8$.
Now our current goal is to analyze the epistemic uncertainty of this model.

\begin{lemma}[Epistemic uncertainty of the classification-equivalent model]
  \label{prop:threshold-uncertainty}
  For fixed $(\mbx, y)$, the residual under $q_t$ is Gaussian and the thresholded success event is a Bernoulli random variable.
  Hence, its epistemic uncertainty is exactly the binary entropy of its success probability, which in the $n, d \to \infty$ limit can be written in terms of the excess risk $\widetilde{\mathcal{L}}(t)$:
  \begin{align}
    \mathrm{EU}(t)
    \approx
    h_2\!\left(
      \operatorname{erf}\left(
        \sqrt{\frac{\epsilon}{4 \widetilde{\mathcal{L}}(t)}}
      \right)
    \right),
  \end{align}
  where $h_2(p) = -p \log p - (1-p)\log(1-p)$ is the binary entropy.
  It is a non-monotone function of the excess risk, and its peak occurs when
  \begin{align}
    \widetilde{\mathcal{L}}(t_{\mathrm{peak}})
    =
    \frac{\epsilon}{4(\operatorname{erf}^{-1}(1/2))^2}
    \approx 1.10\,\epsilon,
  \end{align}
  i.e., when the thresholded success event is maximally ambiguous. By reversing this relation, we can also derive the $t_\mathrm{peak}$ as a function of $\epsilon$ and the risk dynamics.
\end{lemma}

\begin{remark}
  Note that the peak of the epistemic uncertainty does not occur at the same time as the accuracy threshold, but rather before it, when the excess risk is still around $1.10\,\epsilon$ (instead of $\epsilon/8$); this offset is illustrated in the center-right panel of \cref{fig:theory-summary}.
\end{remark}

Finally, we can analyze the delay between the EU peaks on train and test, which is a proxy for the grokking time defined through epistemic uncertainty.

\begin{theorem}[Delay between threshold-uncertainty peaks]
  \label{prop:threshold-eu-peak-delay}
  Assume the construction of \cref{prop:threshold-uncertainty} and the asymptotics described above.
  If the empirical and generalization uncertainty curves peak in this regime, then the
  delay between the two EU peaks satisfies
  \begin{align}
    \Delta t_{\mathrm{EU,peak}}(\lambda)
    =
    t_{\mathrm{peak},\mathrm{gen}} - t_{\mathrm{peak},\mathrm{emp}}
    \simeq
    \frac{
      \log\left(\frac{1}{1 - \sqrt{\lambda}}\right)
    }{
      \frac{n}{\sigma^2}(1 - \sqrt{\lambda})^2 + \frac{1}{\tau^2}
    }.
  \end{align}
\end{theorem}
This scaling is also visible in \cref{fig:theory-summary}, where the separation between the empirical and generalization EU peaks grows markedly as $\lambda \to 1$, and it follows the asymptotic prediction.

\begin{wrapfigure}{r}{0.4\textwidth}
  \vspace{-0.5cm}
  \centering
  \includegraphics[width=0.4\textwidth]{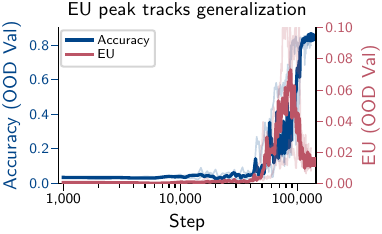}
  \vspace{-0.5cm}
  \caption{
    \textbf{Empirical trajectory of accuracy and epistemic uncertainty.}
    Even in the full transformer model, the EU peak occurs before the accuracy threshold is exceeded, similarly to the asymptotic behavior of the simplified model.
  }
  \label{fig:empirical-trajectory}
\end{wrapfigure}
\paragraph{Takeaway.}
\Cref{prop:threshold-eu-peak-delay} is the main theoretical reason to say that
grokking time can also be defined through epistemic uncertainty.
The point is not that the EU peak occurs at the same absolute time as an accuracy threshold (indeed, the EU peak is happing earlier than the accuracy threshold), but rather that both observables inherit the same train-generalization delay from the same slow edge-of-spectrum modes.
In this simplified model, accuracy and epistemic uncertainty are therefore two different probes of the same underlying mechanism.
This is exactly the qualitative behavior we observe in the empirical analysis.
Indeed, even in the full transformer model on the modulo arithmetic dataset, the EU peak occurs before the accuracy threshold is exceeded, showing that the theoretical picture of the simplified model is compatible with the empirical behavior (\cref{fig:empirical-trajectory}).

\section{Concluding Remarks}
\label{sec:conclusions}

We presented an uncertainty quantification perspective on delayed generalization in in-context learning and showed that epistemic uncertainty provides a sharp, label-free signature of grokking.
In a controlled modular linear arithmetic setting with explicit in-distribution and out-of-distribution splits, we found that grokking coincides with a sudden collapse in epistemic uncertainty, and that this behavior remains informative under changes in task diversity, context length, and context noise.
For both \gls{IVON} and Last-Layer Laplace, this uncertainty signal is robust, while still revealing meaningful differences in out-of-distribution behavior.
We further supported these findings with theorethical results on a simplified Bayesian linear model, showing that delayed generalization and uncertainty peaks are driven by the same underlying spectral mechanism.
To conclude, we want to emphasize that while our findings provide a novel perspective on the grokking phenomenon and its relationship to uncertainty dynamics, they are based on a specific synthetic setting and simplified theoretical model.
Extending these insights to larger models and other language tasks remains an important direction for future work, and the interpretation of uncertainty dynamics in more complex settings should be approached with caution.

\section*{Acknowledgments}
This project was provided with AI computing and storage resources by GENCI at IDRIS thanks to the grant 2025-AD011016811 on the supercomputer Jean Zay's A100/H100 partition.
SR SR acknowledges the support of CIRCALIS AI-HPC facility at EURECOM, with partial funding from French Region Sud.

\bibliographystyle{colm2026}
\bibliography{references}

\clearpage
\appendix

\section{Extended Related Works}
\label{sec:related-works-long}

Grokking and \gls{ICL} are training-inference phenomena that expose a common theme in deep learning: generalization can emerge abruptly, late, and in ways not well predicted by early optimization signals.
Grokking, originally characterized on small algorithmic datasets (notably modular arithmetic), is a delayed generalization regime where test performance remains poor long after training performance saturates, and then transitions sharply to near-perfect generalization \citep{power2022grokking}.
\gls{ICL}, popularized by large transformer language models (e.g., GPT-3), is the ability to adapt to a task at inference time by conditioning on a prompt containing input-output examples, without gradient-based parameter updates \citep{brown2020language}.
Explanations of \gls{ICL} increasingly characterize transformers as implementing recognizable statistical estimators or learning procedures in their forward pass, including ridge regression, gradient descent, and in certain regimes Bayesian estimators \citep{akyurek2023what,von_oswald2023transformers, bai2023transformers,xie2025initialization}.
A Bayesian framing provides a natural language for connecting these threads: pretraining can be interpreted as learning an implicit prior (or an implicit family of priors) over tasks/concepts, while ICL corresponds to approximate Bayesian inference (or Bayesian model averaging) over latent task variables given the prompt \citep{xie2022explanation, zhang2025what}.
In parallel, \gls{BDL} treats uncertainty as first-class, via priors and posteriors over parameters and predictive distributions, and offers tools that make uncertainty trajectories during training and inference measurable \citep{graves2011practical,blundell2015weight,gal2016dropout,lakshminarayanan2017simple, maddox2019simple,ritter2018scalable}.

\paragraph{In-context learning.}
While the theoretical understanding of \gls{ICL} is still developing, a growing body of work has empirically studied various aspects of ICL, including calibration and uncertainty measures under prompting.
\citet{zhang2024study} systematically study calibration under \gls{ICL} across NLP tasks and report that increasing the number of \gls{ICL} examples can initially worsen miscalibration before later improving calibration, and that miscalibration is prominent in low-shot settings.
\citet{ling2024uncertainty} propose a formulation that decomposes predictive uncertainty in \gls{ICL} into epistemic and aleatoric components and introduce estimation methods aimed at plug-and-play uncertainty decomposition for \gls{ICL}.
\citet{wang2025uncertainty} study long-context \gls{ICL} and report that adding more in-context examples can reduce total uncertainty and epistemic uncertainty, while also analyzing confidence evolution across layers.

\paragraph{Grokking.}
Empirical studies of grokking have primarily characterized loss and accuracy dynamics, changes in representations, and mechanistic interpretations of the phenomenon.
\citet{power2022grokking} established the canonical grokking setting on small modular arithmetic tasks, and observed that generalization emerges suddenly after a long period of overfitting.
\citet{liu2022towards} provided an extensive phase-diagram evidence and identified multiple learning phases, emphasizing structured representation emergence as a driver of generalization.
\citet{mohamadi2024why} focused specifically on modular addition and propose a theoretical explanation based on the \gls{NTK} regime, showing that the transition to generalization corresponding to leaving the kernel regime happens only after initial overfitting.
A subset of grokking-adjacent work introduces information-theoretic progress measures, which is conceptually close to uncertainty quantification but not identical to predictive uncertainty.
\citet{clauw2024information} analyze higher-order mutual information (synergy/redundancy) among neurons through training and report distinct phases that can anticipate grokking, framing grokking as an emergent phase transition caused by synergistic interactions.
\citet{golechha2024progress} investigates grokking on real-world classification tasks and proposes progress measures such as activation sparsity, absolute weight entropy, and approximate local circuit complexity that correlate with grokking more strongly than weight norms in that study.

\paragraph{Difference from prior work.}
Closer to our work is the work of \citet{he2024learning}, which extended the classic grokking setting to a larger class of modular linear functions, allowing for a more detailed analysis of the learning dynamics not only in-distribution but also out-of-distribution.
While we build on the same experimental set up, our focus is on the Bayesian learning dynamics and uncertainty trajectories during grokking, rather than characterizing the phases of learning or the representational changes.
Notably, we identify a the condition under which the model can generalize in-domain and out-of-domain, and we analyze how the model's uncertainty evolves as it transitions from overfitting to generalization, providing a novel perspective on the grokking phenomenon.

\section{Description of the methods used}
\label{sec:methods}
\subsection{Improved Variational Online Newton (IVON)}
IVON \citep{shen2024variational} is a scalable, second-order optimizer that makes variational inference tractable for large-scale deep neural networks. Instead of minimizing the empirical risk, IVON directly optimizes the Evidence Lower Bound (ELBO):

\begin{equation}
    \mathcal{L}(q) = \lambda\mathbb{E}_{q(\mbw)}[\bar{\ell}(\mbw)] + \mathbb{D}_{KL}(q(\mbw) || p(\mbw))
\end{equation}

where, $q(\mbw) = \cN(\mbw; \mbmu, \operatorname{diag}(\mbSigma)^2)$ is the diagonal Gaussian posterior approximation, $p(\mbw)$ is the prior, $\bar{\ell}(\mbw)$ is the empirical loss, and $\lambda$ scales the objective.

Previous methods like VON \citep{khan2018fast} relied on a computationally expensive per-example Gauss-Newton Hessian estimate. IVON circumvents this by using the \textit{reparameterization trick} to estimate the Hessian $\hat{\mbh}$ using only the standard minibatch gradient:

\begin{equation}
    \hat{\mbh} \leftarrow \hat{\nabla} \bar{\ell}(\mbw) \odot \frac{\mbw - \mbmu}{\mbSigma^2}
\end{equation}

where $\odot$ is the element-wise product and $\mbw \sim q(\mbw)$. By requiring only a single vector multiplication, IVON's computational cost matches Adam's.

To ensure the Hessian remains positive-definite, IVON applies a Riemannian gradient descent update:

\begin{equation}
    \mbh \leftarrow \beta_2 \mbh + (1 - \beta_2) \hat{\mbh} + \frac{1}{2}(1 - \beta_2)^2 \frac{(\mbh - \hat{\mbh})^2}{\mbh + \delta}
\end{equation}

Here, $\beta_2$ is the Hessian momentum, and $\delta > 0$ acts as a damping parameter corresponding to the prior precision (weight decay). 

The mean $\mbmu$ is updated via a preconditioned Newton-like step, and the standard deviation is tied to the Hessian via $\mbSigma = 1 / \sqrt{\lambda(\mbh + \delta)}$. At inference, IVON quantifies uncertainty through Bayesian Model Averaging (BMA) by sampling $\mbw^{(s)} \sim \cN(\mbmu, \operatorname{diag}(\mbSigma)^2)$.

\subsection{Laplace approximation (LA)}
Laplace approximation \citep{daxberger2021laplace} is a \textit{post-hoc} framework that converts a deterministically trained neural network into a Bayesian model for uncertainty quantification, without modifying the training objective.

LA interprets the final weights trained with $L_2$ regularization as a Maximum A Posteriori (MAP) estimate, $\mbw_{\text{MAP}}$, obtained by minimizing the regularized negative log-posterior:

\begin{equation}
    \mbw_{\text{MAP}} = \arg\min_{\mbw} \left( -\sum_{n=1}^N \log p(y_n \mid f_{\mbw}(\mbx_n)) - \log p(\mbw) \right)
\end{equation}

where $p(\mbw) = \cN(\mbw; \mbzero, \tau^2 \mbI_d)$ is the isotropic Gaussian prior induced by weight decay.

LA constructs a local Gaussian approximation of the intractable true posterior centered at the MAP solution:

\begin{equation}
    p(\mbw \mid \mathcal{D}) \approx \cN(\mbw; \mbw_{\text{MAP}}, \mbSigma)
\end{equation}

The covariance matrix $\mbSigma = \mbH^{-1}$ is the inverse of the Hessian $\mbH$, derived via a second-order Taylor expansion evaluated at $\mbw_{\text{MAP}}$. 

Computing and inverting the full exact Hessian $\mbH \in \bbR^{d \times d}$ is computationally infeasible and often yields an indefinite matrix. Practical LA implementations resolve this through three core approximations:

\begin{enumerate}
    \item Fisher Information Matrix (FIM): Replaces the exact Hessian to guarantee positive semi-definiteness.
    \item Last-Layer Laplace (LLLA): Restricts Bayesian inference to the final linear layer while keeping preceding layers fixed at $\mbw_{\text{MAP}}$, drastically reducing dimensionality.
    \item Kronecker-Factored Approximate Curvature (KFAC): Approximates the Fisher matrix as a Kronecker product $F \approx A \otimes B$ (input and pre-activation gradient covariances) for efficient inversion.
\end{enumerate}

The predictive distribution for a new data point $\mbx^*$ is computed by marginalizing over this approximate posterior:

\begin{equation}
    p(y \mid \mbx^*, \mathcal{D}) \approx \int p(y \mid f_{\mbw}(\mbx^*)) \cN(\mbw; \mbw_{\text{MAP}}, \mbSigma) \, \dd\mbw
\end{equation}

In practice, this integral is approximated via Monte Carlo sampling, producing a distribution of outputs to quantify epistemic uncertainty.

\section{A Simplified Model for Grokking under Variational Inference}
\label{sec:simplified-model}

This section rewrites the derivation in a self-contained form for a
linear-Gaussian teacher-student model. The goal is to isolate the mechanism
that produces distinct training and generalization timescales under Gaussian
variational inference, while keeping all quantities analytically tractable.
Note that while the proof style and some of the intermediate results are similar to those in \citet{levi2024grokking}, the final results and the interpretation are different, as we focus on the variational learning dynamics rather than the gradient flow of the population risk.
Furthermore, the analysis in \citet{levi2024grokking} is mostly focused on loss and accuracy dynamics, while we make a specific focus on the uncertainty dynamics and how they relate to the generalization transition, which is a novel perspective on the grokking phenomenon.

\subsection{Bayesian linear regression}

\begin{definition}[Linear-Gaussian model]
  Let $\mbX \in \bbR^{n \times d}$ be the design matrix, let
  $\mbw \in \bbR^d$ be the parameter vector, and let
  $\mby \in \bbR^n$ be the observed responses. We consider the Bayesian linear
  model
  \begin{align}
    p(\mby \mid \mbw)
    & = \cN(\mby; \mbX \mbw, \sigma^2 \mbI_n),
    \\
    p(\mbw)
    & = \cN(\mbw; \mbzero, \tau^2 \mbI_d),
  \end{align}
  with likelihood variance $\sigma^2 > 0$ and prior variance $\tau^2 > 0$.
  When needed, we also assume the true data-generating process is linear-Gaussian:
  \begin{align}
    y = \mbx^\top \mbw_\star + \varepsilon,
    \qquad
    \E[\varepsilon] = 0,
    \qquad
    \E[\varepsilon^2] = \sigma_\varepsilon^2,
  \end{align}
  with population covariance $\mbSigma_x \equiv \E[\mbx \mbx^\top]$.
\end{definition}

\begin{definition}[Posterior precision and linear term]
  For the model above, define
  \begin{align}
    \mbA & \equiv \frac{\mbX^\top \mbX}{\sigma^2} + \frac{\mbI_d}{\tau^2},
    \\
    \mbb & \equiv \frac{\mbX^\top \mby}{\sigma^2}.
  \end{align}
  Since $\tau^2 > 0$, the matrix $\mbA$ is symmetric positive definite.
\end{definition}

The posterior distribution is for this model is the following Gaussian:
\begin{equation}
  \label{eq:posterior}
  p(\mbw \mid \mby) = \cN(\mbw; \mbmu, \mbSigma),
\end{equation}
where
\begin{align}
  \label{eq:mu_solution}
  \mbSigma & = \mbA^{-1},
  \\
  \mbmu    & = \mbSigma \mbb.
\end{align}

\subsection{Gaussian variational inference as a Wasserstein flow}

\begin{definition}[Gaussian variational family and Bures-Wasserstein flow]
  We approximate the posterior by a time-dependent Gaussian
  \begin{align}
    q_t(\mbw) = \cN(\mbw; \mbmu_t, \mbSigma_t),
  \end{align}
  and study the Bures-Wasserstein gradient flow of the variational objective \citep{lambert2022variational}.
  In the Gaussian family, the flow can be written as
  \begin{align}
    \label{eq:ode-wasserstein}
    \dot \mbmu_t
    & =
    -\E_{q_t}
    \left[
      \gradient_{\mbw_t} \log \frac{q_t(\mbw_t)}{p(\mby, \mbw_t)}
    \right],
    \\
    \dot{\mbSigma}_t
    & =
    -\E_{q_t}
    \Biggl[
      \left(
        \gradient_{\mbw_t} \log \frac{q_t(\mbw_t)}{p(\mby, \mbw_t)}
      \right)
      (\mbw_t - \mbmu_t)^\top
      \nonumber
      +
      (\mbw_t - \mbmu_t)
      \left(
        \gradient_{\mbw_t} \log \frac{q_t(\mbw_t)}{p(\mby, \mbw_t)}
      \right)^\top
    \Biggr].
  \end{align}
\end{definition}

\begin{proposition}[Closed ODEs for the variational mean and covariance]
  \label{prop:closed-odes}
  The Gaussian variational parameters obey
  \begin{align}
    \label{eq:mean_dynamics}
    \dot \mbmu_t
    & = \mbb - \mbA \mbmu_t,
    \\
    \label{eq:cov_dynamics}
    \dot \mbSigma_t
    & = 2 \mbI_d - \mbA \mbSigma_t - \mbSigma_t \mbA.
  \end{align}

  \begin{proof}
    We have
    \begin{align}
      \log q_t(\mbw)
      =
      -\frac{1}{2}(\mbw - \mbmu_t)^\top \mbSigma_t^{-1}(\mbw - \mbmu_t)
      + \mathrm{const},
    \end{align}
    so
    \begin{align}
      \gradient_{\mbw} \log q_t(\mbw)
      & = -\mbSigma_t^{-1} (\mbw - \mbmu_t),
      \\
      \gradient_{\mbw} \log p(\mby, \mbw)
      & = \mbb - \mbA \mbw.
    \end{align}
    Subtracting these two gradients yields
    \begin{equation}
      \label{eq:gradient_log_joint_density}
      -\gradient_{\mbw} \log \frac{q_t(\mbw)}{p(\mby, \mbw)}
      =
      \mbb - \mbA \mbw + \mbSigma_t^{-1} (\mbw - \mbmu_t).
    \end{equation}
    Taking the expectation of \cref{eq:gradient_log_joint_density} under $q_t$
    and using $\E_{q_t}[\mbw] = \mbmu_t$ gives
    \begin{align}
      -\E_{q_t}
      \left[
        \gradient_{\mbw} \log \frac{q_t(\mbw)}{p(\mby, \mbw)}
      \right]
      =
      \mbb - \mbA \mbmu_t,
    \end{align}
    which proves \cref{eq:mean_dynamics}.

    For the covariance dynamics, set $\mbd = \mbw - \mbmu_t$ so that
    $\E_{q_t}[\mbd] = \mbzero$ and $\E_{q_t}[\mbd \mbd^\top] = \mbSigma_t$.
    Then
    \begin{align}
      \E_{q_t}
      \left[
        \left(
          -\gradient_{\mbw} \log \frac{q_t(\mbw)}{p(\mby, \mbw)}
        \right)
        \mbd^\top
      \right]
      & =
      \E_{q_t}
      \left[
        (\mbb - \mbA \mbw)\mbd^\top
      \right]
      +
      \E_{q_t}
      \left[
        \mbSigma_t^{-1} \mbd \mbd^\top
      \right]
      \nonumber \\
      & =
      -\mbA \mbSigma_t + \mbI_d.
    \end{align}
    Indeed, the first term becomes
    \begin{align}
      \E_{q_t}
      \left[
        (\mbb - \mbA (\mbmu_t + \mbd))\mbd^\top
      \right]
      =
      -\mbA \E_{q_t}[\mbd \mbd^\top]
      =
      -\mbA \mbSigma_t,
    \end{align}
    and the second simplifies to
    \begin{align}
      \E_{q_t}
      \left[
        \mbSigma_t^{-1}\mbd \mbd^\top
      \right]
      =
      \mbSigma_t^{-1}\mbSigma_t
      =
      \mbI_d.
    \end{align}
    Plugging this identity and its transpose into
    \cref{eq:ode-wasserstein} yields
    \begin{align}
      \dot \mbSigma_t
      =
      (\mbI_d - \mbA \mbSigma_t)
      +
      (\mbI_d - \mbA \mbSigma_t)^\top
      =
      2 \mbI_d - \mbA \mbSigma_t - \mbSigma_t \mbA,
    \end{align}
    because $\mbA$ is symmetric.
  \end{proof}
\end{proposition}

\begin{theorem}[Explicit variational trajectory and convergence]
  \label{thm:vi-trajectory}
  Let $\mbmu_\infty \equiv \mbA^{-1}\mbb$ and
  $\mbSigma_\infty \equiv \mbA^{-1}$. Then the solution of
  \cref{eq:mean_dynamics,eq:cov_dynamics} is
  \begin{align}
    \label{eq:mean_solution}
    \mbmu_t
    & =
    \mbmu_\infty + e^{-\mbA t}(\mbmu_0 - \mbmu_\infty),
    \\
    \label{eq:cov_solution}
    \mbSigma_t
    & =
    \mbSigma_\infty
    + e^{-\mbA t}(\mbSigma_0 - \mbSigma_\infty)e^{-\mbA t}.
  \end{align}
  In particular, $(\mbmu_t, \mbSigma_t)$ converges exponentially fast to the
  exact posterior parameters from \cref{eq:mu_solution}.

  \begin{proof}
    The fixed point of \cref{eq:mean_dynamics} is $\mbmu_\infty = \mbA^{-1}\mbb$.
    Writing $\delta \mbmu_t \equiv \mbmu_t - \mbmu_\infty$, we obtain
    \begin{align}
      \dot{\delta \mbmu_t} = -\mbA \delta \mbmu_t,
    \end{align}
    hence $\delta \mbmu_t = e^{-\mbA t} \delta \mbmu_0$, which proves
    \cref{eq:mean_solution}.

    Likewise, $\mbSigma_\infty = \mbA^{-1}$ solves the stationary equation $2 \mbI_d - \mbA \mbSigma_\infty - \mbSigma_\infty \mbA = \mbzero$.
    Set $\Delta_t \equiv \mbSigma_t - \mbSigma_\infty$. Then
    \begin{align}
      \dot \Delta_t = -\mbA \Delta_t - \Delta_t \mbA.
    \end{align}
    Since $\mbA$ is symmetric, the function $e^{-\mbA t}\Delta_0 e^{-\mbA t}$ satisfies the same ODE and initial
    condition, proving \cref{eq:cov_solution}.
    Because $\mbA \succ 0$, the matrix exponential decays exponentially, so $\mbmu_t \to \mbmu_\infty$ and $\mbSigma_t \to \mbSigma_\infty$.
    Tese limits are exactly the posterior mean and covariance.
  \end{proof}
\end{theorem}

\subsection{Empirical and generalization risks}

\begin{definition}[Empirical and population risks]
  For a deterministic predictor $\mbw$, define the squared-loss risks
  \begin{align}
    \cR_{\text{emp}}(\mbw)
    & = \frac{1}{2n}\norm{\mby - \mbX \mbw}^2,
    \\
    \cR_{\text{gen}}(\mbw)
    & = \frac{1}{2}\E_{p(\mbx, y)} (y - \mbx^\top \mbw)^2.
  \end{align}
  Along the Gaussian variational trajectory, we are interested in the expected
  risks
  \begin{align}
    \mathcal{L}_{\text{emp}}(t)
    & \equiv \E_{q_t}\cR_{\text{emp}}(\mbw),
    \\
    \mathcal{L}_{\text{gen}}(t)
    & \equiv \E_{q_t}\cR_{\text{gen}}(\mbw).
  \end{align}
\end{definition}

\begin{proposition}[Exact risks under a Gaussian predictor]
  \label{prop:gaussian-risks}
  For any Gaussian predictor $q_t(\mbw) = \cN(\mbmu_t, \mbSigma_t)$,
  \begin{align}
    \label{eq:emp_risk_gaussian}
    \mathcal{L}_{\text{emp}}(t)
    & =
    \frac{1}{2n}
    \left(
      \norm{\mby - \mbX \mbmu_t}^2
      + \Tr(\mbX^\top \mbX \mbSigma_t)
    \right),
    \\
    \label{eq:gen_risk_gaussian}
    \mathcal{L}_{\text{gen}}(t)
    & =
    \frac{1}{2}
    \left(
      (\mbmu_t - \mbw_\star)^\top
      \mbSigma_x
      (\mbmu_t - \mbw_\star)
      +
      \Tr(\mbSigma_x \mbSigma_t)
      +
      \sigma_\varepsilon^2
    \right).
  \end{align}

  \begin{proof}
    Let $\mbw \sim \cN(\mbmu_t, \mbSigma_t)$. Then
    \begin{align}
      \E_{q_t}\norm{\mby - \mbX \mbw}^2
      =
      \norm{\mby - \mbX \mbmu_t}^2
      + \Tr(\mbX^\top \mbX \mbSigma_t),
    \end{align}
    which proves \cref{eq:emp_risk_gaussian}. For the population risk, write
    $y = \mbx^\top \mbw_\star + \varepsilon$ and use independence of
    $\varepsilon$, $\mbx$, and $\mbw$:
    \begin{align}
      \E (y - \mbx^\top \mbw)^2
      & =
      \E\bigl[\mbx^\top (\mbw_\star - \mbw)\bigr]^2 + \sigma_\varepsilon^2
      \nonumber \\
      & =
      \E_{\mbx}
      \left[
        (\mbw_\star - \mbmu_t)^\top \mbx \mbx^\top (\mbw_\star - \mbmu_t)
      \right]
      +
      \E_{\mbx}\Tr(\mbx \mbx^\top \mbSigma_t)
      +
      \sigma_\varepsilon^2
      \nonumber \\
      & =
      (\mbmu_t - \mbw_\star)^\top \mbSigma_x (\mbmu_t - \mbw_\star)
      + \Tr(\mbSigma_x \mbSigma_t)
      + \sigma_\varepsilon^2.
    \end{align}
    Multiplying by $1/2$ gives \cref{eq:gen_risk_gaussian}.
  \end{proof}
\end{proposition}

\begin{corollary}[Exact risk decomposition along the VI flow]
  \label{cor:vi-risk-decomposition}
  Define
  \begin{align}
    \mbE_t
    & \equiv e^{-\mbA t},
    &
    \delta\mbmu_0
    & \equiv \mbmu_0 - \mbmu_\infty,
    &
    \Delta\mbSigma_0
    & \equiv \mbSigma_0 - \mbSigma_\infty.
  \end{align}
  Let $\mbr_\infty \equiv \mby - \mbX \mbmu_\infty$ and
  $\mbb_\infty \equiv \mbmu_\infty - \mbw_\star$. Then
  \begin{align}
    \label{eq:emp_risk_vi_solution}
    \mathcal{L}_{\text{emp}}(t)
    & =
    \mathcal{L}_{\text{emp},\infty}
    - \frac{1}{n}\mbr_\infty^\top \mbX \mbE_t \delta\mbmu_0
    + \frac{1}{2n}
    \delta\mbmu_0^\top \mbE_t \mbX^\top \mbX \mbE_t \delta\mbmu_0
    + \frac{1}{2n}\Tr(\mbX^\top \mbX \mbE_t \Delta\mbSigma_0 \mbE_t),
    \\
    \label{eq:gen_risk_vi_solution}
    \mathcal{L}_{\text{gen}}(t)
    & =
    \mathcal{L}_{\text{gen},\infty}
    + \mbb_\infty^\top \mbSigma_x \mbE_t \delta\mbmu_0
    + \frac{1}{2}
    \delta\mbmu_0^\top \mbE_t \mbSigma_x \mbE_t \delta\mbmu_0
    + \frac{1}{2}\Tr(\mbSigma_x \mbE_t \Delta\mbSigma_0 \mbE_t),
  \end{align}
  where
  \begin{align}
    \mathcal{L}_{\text{emp},\infty}
    & \equiv
    \frac{1}{2n}
    \left(
      \norm{\mbr_\infty}^2 + \Tr(\mbX^\top \mbX \mbSigma_\infty)
    \right),
    \\
    \mathcal{L}_{\text{gen},\infty}
    & \equiv
    \frac{1}{2}
    \left(
      \mbb_\infty^\top \mbSigma_x \mbb_\infty
      + \Tr(\mbSigma_x \mbSigma_\infty)
      + \sigma_\varepsilon^2
    \right).
  \end{align}

  \begin{proof}
    Substitute \cref{eq:mean_solution,eq:cov_solution} into
    \cref{eq:emp_risk_gaussian,eq:gen_risk_gaussian} and expand the resulting
    quadratic forms. The empirical and generalization fixed-point terms are the
    parts independent of $\mbE_t$; the linear and quadratic time-dependent
    contributions are exactly those displayed above.
  \end{proof}
\end{corollary}

\begin{corollary}[Matched interpolation regime]
  \label{cor:matched-interpolation}
  Assume $\mbSigma_x = \mbI_d$ and that the stationary predictor is
  well-specified in the sense that
  \begin{align}
    \mbmu_\infty = \mbw_\star,
    \qquad
    \mbr_\infty = \mbzero.
  \end{align}
  Define the excess risks
  \begin{align}
    \widetilde{\mathcal{L}}_{\text{emp}}(t)
    & \equiv
    \mathcal{L}_{\text{emp}}(t) - \mathcal{L}_{\text{emp},\infty},
    \\
    \widetilde{\mathcal{L}}_{\text{gen}}(t)
    & \equiv
    \mathcal{L}_{\text{gen}}(t) - \mathcal{L}_{\text{gen},\infty}.
  \end{align}
  Then
  \begin{align}
    \label{eq:emp_excess_risk_interpolation}
    \widetilde{\mathcal{L}}_{\text{emp}}(t)
    & =
    \frac{1}{2n}
    \delta\mbmu_0^\top \mbE_t \mbX^\top \mbX \mbE_t \delta\mbmu_0
    + \frac{1}{2n}\Tr(\mbX^\top \mbX \mbE_t \Delta\mbSigma_0 \mbE_t),
    \\
    \label{eq:gen_excess_risk_interpolation}
    \widetilde{\mathcal{L}}_{\text{gen}}(t)
    & =
    \frac{1}{2}\delta\mbmu_0^\top \mbE_t^2 \delta\mbmu_0
    + \frac{1}{2}\Tr(\mbE_t \Delta\mbSigma_0 \mbE_t).
  \end{align}

  \begin{proof}
    In \cref{eq:emp_risk_vi_solution,eq:gen_risk_vi_solution}, the linear terms
    vanish because $\mbr_\infty = \mbzero$ and $\mbb_\infty = \mbzero$, while
    $\mbSigma_x = \mbI_d$ simplifies the quadratic and trace terms.
  \end{proof}
\end{corollary}

\subsection{Spectral reduction and random-matrix asymptotics}

\begin{lemma}[Spectral form of the excess risks]
  \label{lem:spectral-risks}
  Let
  \begin{align}
    \widehat{\mbSigma}
    \equiv
    \frac{1}{n}\mbX^\top \mbX
    =
    \mbU \operatorname{diag}(\nu_1, \dots, \nu_d)\mbU^\top.
  \end{align}
  Then
  \begin{align}
    \mbA
    =
    \mbU \operatorname{diag}(a_1, \dots, a_d)\mbU^\top,
    \qquad
    a_i = \frac{n}{\sigma^2}\nu_i + \frac{1}{\tau^2}.
  \end{align}
  Writing
  \begin{align}
    \bar{\mbd}_0
    & \equiv \mbU^\top \delta\mbmu_0,
    &
    \bar{\mbS}_0
    & \equiv \mbU^\top \Delta\mbSigma_0 \mbU,
  \end{align}
  the excess risks in \cref{eq:emp_excess_risk_interpolation,eq:gen_excess_risk_interpolation}
  become
  \begin{align}
    \label{eq:emp_risk_spectral_sum}
    \widetilde{\mathcal{L}}_{\text{emp}}(t)
    & =
    \frac{1}{2}
    \sum_{i=1}^d
    \nu_i \bigl(\bar d_{0,i}^2 + (\bar{\mbS}_0)_{ii}\bigr)e^{-2 a_i t},
    \\
    \label{eq:gen_risk_spectral_sum}
    \widetilde{\mathcal{L}}_{\text{gen}}(t)
    & =
    \frac{1}{2}
    \sum_{i=1}^d
    \bigl(\bar d_{0,i}^2 + (\bar{\mbS}_0)_{ii}\bigr)e^{-2 a_i t}.
  \end{align}

  \begin{proof}
    $\widehat{\mbSigma}$ and $\mbA$ are diagonalized by the same orthogonal matrix $\mbU$.
    In that basis,
    \begin{align}
      \mbE_t = e^{-\mbA t}
      =
      \mbU \operatorname{diag}(e^{-a_1 t}, \dots, e^{-a_d t}) \mbU^\top.
    \end{align}
    Substituting this into
    \cref{eq:emp_excess_risk_interpolation,eq:gen_excess_risk_interpolation}
    immediately yields the stated sums.
  \end{proof}
\end{lemma}

\begin{assumption}[Self-averaging initialization]
  \label{ass:self-averaging}
  In the large-system limit, assume that the initial mode weights are
  asymptotically uncorrelated with the eigenbasis of $\widehat{\mbSigma}$, so
  that
  \begin{align}
    \bar d_{0,i}^2 + (\bar{\mbS}_0)_{ii}
    \approx
    \frac{\norm{\delta\mbmu_0}^2 + \Tr(\Delta\mbSigma_0)}{d}
  \end{align}
  uniformly over $i$.
\end{assumption}

\begin{proposition}[Self-averaged loss formulas]
  \label{prop:self-averaged-losses}
  Under \cref{ass:self-averaging}, if
  \begin{align}
    C_0 \equiv \frac{1}{2}
    \left(
      \norm{\delta\mbmu_0}^2 + \Tr(\Delta\mbSigma_0)
    \right),
  \end{align}
  then
  \begin{align}
    \widetilde{\mathcal{L}}_{\text{emp}}(t)
    & \approx
    C_0 \frac{1}{d}\sum_{i=1}^d
    \nu_i
    \exp\left[
      -2 \left(
        \frac{n}{\sigma^2}\nu_i + \frac{1}{\tau^2}
      \right)t
    \right],
    \\
    \widetilde{\mathcal{L}}_{\text{gen}}(t)
    & \approx
    C_0 \frac{1}{d}\sum_{i=1}^d
    \exp\left[
      -2 \left(
        \frac{n}{\sigma^2}\nu_i + \frac{1}{\tau^2}
      \right)t
    \right].
  \end{align}
  If, moreover, $d,n \to \infty$ with aspect ratio $\lambda = d/n < 1$, then
  \begin{align}
    \label{eq:vi_emp_loss_anal}
    \widetilde{\mathcal{L}}_{\text{emp}}(t)
    & \approx
    C_0\,
    \E_{\nu \sim \mathrm{MP}(\lambda)}
    \left[
      \nu \exp\left(
        -2 \left(
          \frac{n}{\sigma^2}\nu + \frac{1}{\tau^2}
        \right)t
      \right)
    \right],
    \\
    \label{eq:vi_gen_loss_anal}
    \widetilde{\mathcal{L}}_{\text{gen}}(t)
    & \approx
    C_0\,
    \E_{\nu \sim \mathrm{MP}(\lambda)}
    \left[
      \exp\left(
        -2 \left(
          \frac{n}{\sigma^2}\nu + \frac{1}{\tau^2}
        \right)t
      \right)
    \right].
  \end{align}

  \begin{proof}
    Replace the mode-dependent prefactors in
    \cref{eq:emp_risk_spectral_sum,eq:gen_risk_spectral_sum} by the common
    self-averaged value from \cref{ass:self-averaging}. This yields the finite
    sums. The Marchenko-Pastur formulas follow by replacing the empirical
    spectral average with its large-$d,n$ limit.
  \end{proof}
\end{proposition}

\begin{theorem}[Late-time asymptotics and grokking delay]
  \label{app:thm:grokking-delay}
  Assume the setting of \cref{prop:self-averaged-losses} with $\lambda < 1$,
  and let
  \begin{align}
    \nu_\pm = (1 \pm \sqrt{\lambda})^2,
    \qquad
    \kappa
    =
    2\left(
      \frac{n}{\sigma^2}(1 - \sqrt{\lambda})^2 + \frac{1}{\tau^2}
    \right).
  \end{align}
  Then, as $t \to \infty$,
  \begin{align}
    \label{eq:vi_emp_late_time}
    \widetilde{\mathcal{L}}_{\text{emp}}(t)
    & \sim
    K_{\text{emp}}\, t^{-3/2} e^{-\kappa t},
    \\
    \label{eq:vi_gen_late_time}
    \widetilde{\mathcal{L}}_{\text{gen}}(t)
    & \sim
    K_{\text{gen}}\, t^{-3/2} e^{-\kappa t},
  \end{align}
  with prefactor ratio
  \begin{equation}
    \label{eq:vi_prefactor_ratio}
    \frac{K_{\text{gen}}}{K_{\text{emp}}}
    = \frac{1}{\nu_-}
    = \frac{1}{(1 - \sqrt{\lambda})^2}.
  \end{equation}
  If the grokking time is defined by the $95\%$ accuracy threshold
  \begin{align}
    \widetilde{\mathcal{L}}(t^\star) = \frac{\epsilon}{8},
  \end{align}
  then the leading train-generalization delay is
  \begin{equation}
    \label{eq:vi_grokking_time}
    \Delta t_{\text{grok}}
    \equiv
    t^\star_{\text{gen}} - t^\star_{\text{emp}}
    \simeq
    \frac{
      \log\left(\frac{1}{1 - \sqrt{\lambda}}\right)
    }{
      \frac{n}{\sigma^2}(1 - \sqrt{\lambda})^2 + \frac{1}{\tau^2}
    }.
  \end{equation}

  \begin{proof}
    The Marchenko-Pastur density on $[\nu_-, \nu_+]$ behaves near its lower
    edge as
    \begin{align}
      p_{\mathrm{MP}}(\nu) \propto \sqrt{\nu - \nu_-}.
    \end{align}
    Write $\nu = \nu_- + u$. In
    \cref{eq:vi_emp_loss_anal,eq:vi_gen_loss_anal}, the exponential factor is
    $e^{-\kappa t} e^{-2(n/\sigma^2) u t}$. Hence both integrals are governed
    by the same edge contribution
    \begin{align}
      e^{-\kappa t}
      \int_0^\infty u^{1/2} e^{-2(n/\sigma^2)u t}\dd u
      \propto
      t^{-3/2} e^{-\kappa t},
    \end{align}
    which proves \cref{eq:vi_emp_late_time,eq:vi_gen_late_time}. The empirical
    loss contains an extra factor $\nu = \nu_- + u$, which contributes only
    the multiplicative constant $\nu_-$ at leading order; this yields
    \cref{eq:vi_prefactor_ratio}.

    To obtain the delay, solve
    $K t^{-3/2} e^{-\kappa t} = \epsilon/8$ asymptotically for $t$. The
    subleading $t^{-3/2}$ correction is the same for the empirical and
    generalization losses, so the leading difference between the two solutions
    is
    \begin{align}
      \Delta t_{\text{grok}}
      \simeq
      \frac{\log(K_{\text{gen}}/K_{\text{emp}})}{\kappa}.
    \end{align}
    Using \cref{eq:vi_prefactor_ratio} and the definition of $\kappa$ gives
    \cref{eq:vi_grokking_time}.
  \end{proof}
\end{theorem}

\subsection{Accuracy and uncertainty of the sampled predictor}

\begin{proposition}[Exact accuracy of the sampled VI predictor]
  \label{prop:sampled-accuracy}
  Let $\mbw \sim q_t$. For a fixed datapoint $(\mbx, y)$, define
  \begin{align}
    m_t(\mbx, y)
    & \equiv y - \mbx^\top \mbmu_t,
    &
    v_t(\mbx)
    & \equiv \mbx^\top \mbSigma_t \mbx.
  \end{align}
  Then the residual
  \begin{align}
    r_t(\mbx, y; \mbw) \equiv y - \mbx^\top \mbw
  \end{align}
  is conditionally Gaussian under $q_t$, and for any tolerance $\epsilon > 0$
  the empirical and population accuracies are
  \begin{align}
    \label{eq:emp_accuracy_sampled_predictor}
    \mathcal{A}_{\text{emp}}(t)
    & =
    \frac{1}{n}\sum_{i=1}^n
    \left[
      \Phi\left(
        \frac{\sqrt{\epsilon} - m_i(t)}{\sqrt{v_i(t)}}
      \right)
      -
      \Phi\left(
        \frac{-\sqrt{\epsilon} - m_i(t)}{\sqrt{v_i(t)}}
      \right)
    \right],
    \\
    \label{eq:gen_accuracy_sampled_predictor}
    \mathcal{A}_{\text{gen}}(t)
    & =
    \E_{p(\mbx)}
    \left[
      \Phi\left(
        \frac{\sqrt{\epsilon} - \mbx^\top(\mbw_\star - \mbmu_t)}
        {\sqrt{\sigma_\varepsilon^2 + \mbx^\top \mbSigma_t \mbx}}
      \right)
      -
      \Phi\left(
        \frac{-\sqrt{\epsilon} - \mbx^\top(\mbw_\star - \mbmu_t)}
        {\sqrt{\sigma_\varepsilon^2 + \mbx^\top \mbSigma_t \mbx}}
      \right)
    \right],
  \end{align}
  where $m_i(t) = y_i - \mbx_i^\top \mbmu_t$ and $v_i(t) = \mbx_i^\top \mbSigma_t \mbx_i$, and $\Phi$ is the standard Gaussian CDF.

  \begin{proof}
    Since $\mbw \sim \cN(\mbmu_t, \mbSigma_t)$, the scalar random variable
    $\mbx^\top \mbw$ is Gaussian with mean $\mbx^\top \mbmu_t$ and variance
    $\mbx^\top \mbSigma_t \mbx$. Therefore
    \begin{align}
      r_t(\mbx, y; \mbw)\mid (\mbx, y)
      \sim
      \cN(m_t(\mbx, y), v_t(\mbx)).
    \end{align}
    The formulas follow by evaluating the Gaussian probability of the interval
    $[-\sqrt{\epsilon}, \sqrt{\epsilon}]$. For the generalization expression,
    use $y = \mbx^\top \mbw_\star + \varepsilon$ and absorb the label noise into
    the conditional variance.
  \end{proof}
\end{proposition}

\begin{proposition}[Entropy decomposition of predictive uncertainty]
  \label{prop:entropy-decomposition}
  Assume the predictive model
  \begin{align}
    y \mid \mbx, \mbw \sim \cN(\mbx^\top \mbw, \sigma_\varepsilon^2).
  \end{align}
  Then the predictive entropy decomposes as
  \begin{align}
    \label{eq:entropy_decomposition}
    \mathbb{H}[y \mid \mbx, \mathcal{D}, t]
    =
    \E_{q_t(\mbw)} \mathbb{H}[y \mid \mbx, \mbw]
    +
    \mathbb{I}_t(y; \mbw \mid \mbx, \mathcal{D}).
  \end{align}
  Moreover, the predictive distribution is
  \begin{align}
    \label{eq:predictive_distribution_vi}
    p_t(y \mid \mbx, \mathcal{D})
    =
    \cN\left(
      y;
      \mbx^\top \mbmu_t,
      \sigma_\varepsilon^2 + \mbx^\top \mbSigma_t \mbx
    \right),
  \end{align}
  and the pointwise epistemic uncertainty is
  \begin{align}
    \label{eq:epistemic_entropy_pointwise}
    \text{EU}(\mbx, t)
    =
    \frac{1}{2}\log\left(
      1 + \frac{\mbx^\top \mbSigma_t \mbx}{\sigma_\varepsilon^2}
    \right).
  \end{align}
  Using the explicit covariance trajectory from \cref{eq:cov_solution}, this
  can also be written as
  \begin{align}
    \label{eq:epistemic_entropy_vi_solution}
    \text{EU}(\mbx, t)
    =
    \frac{1}{2}\log\left(
      1 + \frac{
        \mbx^\top \mbSigma_\infty \mbx
        +
        \mbx^\top \mbE_t \Delta\mbSigma_0 \mbE_t \mbx
      }{\sigma_\varepsilon^2}
    \right).
  \end{align}

  \begin{proof}
    The entropy decomposition is the standard identity
    $H(Y \mid X) = \E H(Y \mid X, W) + I(Y; W \mid X)$. Since both
    $p(y \mid \mbx, \mbw)$ and $q_t(\mbw)$ are Gaussian, marginalizing over
    $\mbw$ gives \cref{eq:predictive_distribution_vi}. The differential entropy
    of a Gaussian $\cN(0, s^2)$ is $\frac{1}{2}\log(2 \pi e s^2)$, so the
    aleatoric term is $\frac{1}{2}\log(2 \pi e \sigma_\varepsilon^2)$ and the
    epistemic contribution is the difference between the predictive and
    aleatoric entropies, which yields \cref{eq:epistemic_entropy_pointwise}.
  \end{proof}
\end{proposition}

\begin{remark}[Late-time accuracy closure]
  \label{rem:accuracy-closure}
  In the matched interpolation regime, the late-time residual is approximately
  centered and its variance is controlled by the excess risk:
  \begin{align}
    \E[r_t^2]
    \approx
    2 \widetilde{\mathcal{L}}(t).
  \end{align}
  Approximating the residual by a single effective Gaussian therefore gives
  \begin{align}
    \label{eq:vi_accuracy_erf_closure}
    \mathcal{A}_{\text{emp}/\text{gen}}(t)
    \approx
    \operatorname{erf}\left(
      \sqrt{
        \frac{\epsilon}{4 \widetilde{\mathcal{L}}_{\text{emp}/\text{gen}}(t)}
      }
    \right).
  \end{align}
  The $95\%$ grokking threshold used in \cref{app:thm:grokking-delay} then
  corresponds to $\widetilde{\mathcal{L}}(t^\star) = \epsilon / 8$ because
  $\operatorname{erf}(\sqrt{2}) \approx 0.95$.
\end{remark}

\begin{lemma}[Trace dynamics of the covariance]
  \label{lem:trace-dynamics}
  The covariance trace obeys
  \begin{align}
    \label{eq:trace_sigma_derivative}
    \frac{\dd}{\dd t}\Tr(\mbSigma_t)
    =
    2 \Tr(\mbI_d - \mbA \mbSigma_t).
  \end{align}
  In the eigenbasis of $\mbA$, if
  $\mbSigma_t = \mbSigma_\infty + \mbE_t \Delta\mbSigma_0 \mbE_t$ and
  $\Delta s_i$ denotes the $i$-th diagonal entry of $\mbU^\top \Delta\mbSigma_0 \mbU$,
  then
  \begin{align}
    \frac{\dd}{\dd t}\Tr(\mbSigma_t)
    =
    -2 \sum_{i=1}^d a_i \Delta s_i e^{-2 a_i t}.
  \end{align}

  \begin{proof}
    Taking the trace of \cref{eq:cov_dynamics} gives
    \begin{align}
      \frac{\dd}{\dd t}\Tr(\mbSigma_t)
      =
      2 \Tr(\mbI_d) - \Tr(\mbA \mbSigma_t) - \Tr(\mbSigma_t \mbA)
      =
      2 \Tr(\mbI_d - \mbA \mbSigma_t),
    \end{align}
    using cyclicity of the trace. The spectral formula follows by substituting
    the explicit solution from \cref{eq:cov_solution}.
  \end{proof}
\end{lemma}

\begin{lemma}[Population-averaged epistemic uncertainty]
  \label{thm:population-epistemic}
  The exact population-averaged epistemic uncertainty is
  \begin{align}
    \label{eq:population_epistemic_entropy_exact}
    \text{EU}(t)
    \equiv
    \E_{p(\mbx)} \text{EU}(\mbx, t)
    =
    \frac{1}{2}\E_{p(\mbx)}
    \log\left(
      1 + \frac{\mbx^\top \mbSigma_t \mbx}{\sigma_\varepsilon^2}
    \right).
  \end{align}
  Under isotropic Gaussian design $\mbx \sim \cN(\mbzero, \mbI_d)$ and the
  same self-averaging closure used above,
  \begin{align}
    \label{eq:population_epistemic_entropy_trace_closure}
    \text{EU}(t)
    \approx
    \frac{1}{2}\log\left(
      1 + \frac{\Tr(\mbSigma_t)}{\sigma_\varepsilon^2}
    \right).
  \end{align}
  If $\mbSigma_0 \prec \mbSigma_\infty$, this approximate epistemic curve is
  monotonically increasing; if $\mbSigma_0 \succ \mbSigma_\infty$, it is
  monotonically decreasing; and if $\mbSigma_0 = \mbSigma_\infty$, it is
  constant.

  \begin{proof}
    The exact identity is just the average of
    \cref{eq:epistemic_entropy_pointwise}. Under isotropic Gaussian design, the
    quadratic form $\mbx^\top \mbSigma_t \mbx$ concentrates around its trace in
    high dimension, yielding \cref{eq:population_epistemic_entropy_trace_closure}.
    Because $u \mapsto \frac{1}{2}\log(1 + u/\sigma_\varepsilon^2)$ is
    increasing, the monotonicity of the approximate entropy curve matches the
    monotonicity of $\Tr(\mbSigma_t)$. The explicit solution
    \cref{eq:cov_solution} shows that each spectral mode evolves as
    $a_i^{-1} + \Delta s_i e^{-2 a_i t}$, so the sign of $\Delta s_i$ determines
    whether that mode increases or decreases. If
    $\mbSigma_0 \prec \mbSigma_\infty$, then all $\Delta s_i < 0$ and
    \cref{lem:trace-dynamics} implies that $\Tr(\mbSigma_t)$ increases
    monotonically; the opposite case is analogous, and equality gives a
    constant curve.
  \end{proof}
\end{lemma}

\begin{corollary}[Isotropic initialization]
  \label{cor:isotropic-init}
  If $\mbSigma_0 = s_0 \mbI_d$, then
  \begin{align}
    \frac{1}{d}\Tr(\mbSigma_t)
    \approx
    \E_{\nu \sim \mathrm{MP}(\lambda)}
    \left[
      \frac{1}{a(\nu)}
      +
      \left(
        s_0 - \frac{1}{a(\nu)}
      \right)e^{-2 a(\nu) t}
    \right],
    \qquad
    a(\nu) = \frac{n}{\sigma^2}\nu + \frac{1}{\tau^2}.
  \end{align}
  In particular, when $\mbSigma_0 = \mbzero$,
  \begin{align}
    \frac{1}{d}\Tr(\mbSigma_t)
    \approx
    \E_{\nu \sim \mathrm{MP}(\lambda)}
    \left[
      \frac{1 - e^{-2 a(\nu) t}}{a(\nu)}
    \right].
  \end{align}

  \begin{proof}
    Under isotropic initialization, $\Delta\mbSigma_0$ is diagonal in every
    eigenbasis, so the spectral representation of \cref{eq:cov_solution}
    reduces to a scalar mode-by-mode formula. Averaging over the empirical
    spectrum and then passing to the Marchenko-Pastur limit gives the result.
  \end{proof}
\end{corollary}

\begin{proposition}[Epistemic uncertainty of the threshold model]
  \label{app:prop:threshold-uncertainty}
  Define the thresholded success indicator
  \begin{align}
    Z_\epsilon(\mbx, y; \mbw)
    \equiv
    \mathbbm{1}\{(y - \mbx^\top \mbw)^2 \le \epsilon\}.
  \end{align}
  For fixed $(\mbx, y)$, the success probability under $q_t$ is
  \begin{align}
    \label{eq:threshold_success_probability}
    p_t(\mbx, y)
    & =
    \Phi\left(
      \frac{\sqrt{\epsilon} - m_t(\mbx, y)}{\sqrt{v_t(\mbx)}}
    \right)
    -
    \Phi\left(
      \frac{-\sqrt{\epsilon} - m_t(\mbx, y)}{\sqrt{v_t(\mbx)}}
    \right),
  \end{align}
  and
  \begin{align}
    Z_\epsilon(\mbx, y; \mbw) \mid (\mbx, y, \mathcal{D}, t)
    \sim
    \mathrm{Bernoulli}(p_t(\mbx, y)).
  \end{align}
  Hence the threshold-model epistemic uncertainty is exactly
  \begin{align}
    \label{eq:threshold_epistemic_uncertainty}
    \text{EU}_t(Z_\epsilon; \mbw \mid \mbx, y, \mathcal{D})
    =
    h_2\bigl(p_t(\mbx, y)\bigr),
  \end{align}
  where $h_2(p) = -p \log p - (1-p)\log(1-p)$ is the binary entropy.
  In the late-time self-averaging regime,
  \begin{align}
    \label{eq:threshold_success_probability_closure}
    p_t
    \approx
    \operatorname{erf}\left(
      \sqrt{
        \frac{\epsilon}{4 \widetilde{\mathcal{L}}(t)}
      }
    \right),
  \end{align}
  and therefore
  \begin{align}
    \label{eq:threshold_epistemic_uncertainty_closure}
    \text{EU}_{\mathrm{th}}(t)
    \approx
    h_2\!\left(
      \operatorname{erf}\left(
        \sqrt{
          \frac{\epsilon}{4 \widetilde{\mathcal{L}}(t)}
        }
      \right)
    \right).
  \end{align}

  \begin{proof}
    The exact success probability is the Gaussian interval probability from
    \cref{prop:sampled-accuracy}. Conditional on $\mbw$, the variable
    $Z_\epsilon$ is deterministic, so its conditional entropy vanishes. The
    predictive distribution is therefore Bernoulli with entropy
    $h_2(p_t(\mbx, y))$, and the full uncertainty is epistemic. The closure
    formula follows from the same effective-Gaussian approximation that relates
    the residual variance to the excess risk at late times.
  \end{proof}
\end{proposition}

\begin{remark}[Threshold uncertainty versus grokking]
  The binary entropy $h_2(p)$ is maximized at $p = 1/2$, so the threshold-model
  epistemic uncertainty typically follows a non-monotone trajectory: it rises
  while predictions are ambiguous and then falls as the model becomes
  confidently correct. Using \cref{eq:threshold_success_probability_closure},
  the peak occurs when
  \begin{align}
    \label{eq:threshold_eu_peak_condition}
    \widetilde{\mathcal{L}}(t_{\mathrm{peak}})
    =
    \frac{\epsilon}{4(\operatorname{erf}^{-1}(1/2))^2},
  \end{align}
  and numerically
  \begin{align}
    \frac{1}{4(\operatorname{erf}^{-1}(1/2))^2}
    \approx 1.10.
  \end{align}
  Thus the peak occurs when the loss is still of order $\epsilon$, which is much
  earlier than the $95\%$ grokking threshold
  $\widetilde{\mathcal{L}}(t^\star) = \epsilon/8$. Thus the entropy peak and
  the grokking delay are distinct observables, even though both are controlled
  by the same slow edge-of-spectrum modes in the Marchenko-Pastur regime.
\end{remark}

\begin{theorem}[Delay between threshold-uncertainty peaks]
  \label{app:prop:threshold-eu-peak-delay}
  Assume the setting of \cref{app:thm:grokking-delay,app:prop:threshold-uncertainty},
  and suppose the empirical and generalization threshold-model uncertainty peaks
  occur in the late-time regime where
  \cref{eq:threshold_epistemic_uncertainty_closure,eq:vi_emp_late_time,eq:vi_gen_late_time}
  apply. Then
  \begin{align}
    \Delta t_{\text{EU peak}}(\lambda)
    & \equiv
    t_{\mathrm{peak},\mathrm{gen}} - t_{\mathrm{peak},\mathrm{emp}}
    \nonumber \\
    & \simeq
    \frac{
      \log\left(\frac{1}{1 - \sqrt{\lambda}}\right)
    }{
      \frac{n}{\sigma^2}(1 - \sqrt{\lambda})^2 + \frac{1}{\tau^2}
    },
    \label{eq:threshold_eu_peak_delay_lambda}
  \end{align}
  up to threshold-dependent subleading corrections.

  \begin{proof}
    For $\alpha \in \{\mathrm{emp}, \mathrm{gen}\}$, write
    \begin{align}
      h_{\alpha}(t)
      & \approx
      h_2\bigl(p_{\alpha}(t)\bigr),
      &
      p_{\alpha}(t)
      & \approx
      \operatorname{erf}\left(
        \sqrt{
          \frac{\epsilon}{4 \widetilde{\mathcal{L}}_{\alpha}(t)}
        }
      \right),
    \end{align}
    as in \cref{eq:threshold_epistemic_uncertainty_closure}. Since
    $h_2'(p) = \log\!\bigl(\frac{1-p}{p}\bigr)$, the binary entropy is uniquely
    maximized at $p = 1/2$. Moreover, the map
    \begin{align}
      \widetilde{\mathcal{L}}
      \longmapsto
      \operatorname{erf}\left(
        \sqrt{
          \frac{\epsilon}{4 \widetilde{\mathcal{L}}}
        }
      \right)
    \end{align}
    is strictly decreasing in $\widetilde{\mathcal{L}}$, while in the late-time
    regime \cref{eq:vi_emp_late_time,eq:vi_gen_late_time} imply that each loss is
    strictly decreasing for large $t$. Hence each late-time uncertainty curve
    reaches its maximum exactly when $p_{\alpha}(t)=1/2$, namely when
    \begin{align}
      \widetilde{\mathcal{L}}_{\alpha}(t_{\mathrm{peak},\alpha})
      =
      c_{\mathrm{peak}}
      \equiv
      \frac{\epsilon}{4(\operatorname{erf}^{-1}(1/2))^2},
    \end{align}
    which is \cref{eq:threshold_eu_peak_condition}.

    Now use the late-time asymptotics from \cref{app:thm:grokking-delay}. For each
    $\alpha \in \{\mathrm{emp}, \mathrm{gen}\}$,
    \begin{align}
      \widetilde{\mathcal{L}}_{\alpha}(t)
      \sim
      K_{\alpha}\, t^{-3/2} e^{-\kappa t}.
    \end{align}
    Therefore the peak time satisfies
    \begin{align}
      c_{\mathrm{peak}}
      =
      K_{\alpha}\, t_{\mathrm{peak},\alpha}^{-3/2}
      e^{-\kappa t_{\mathrm{peak},\alpha}}
      \bigl(1 + o(1)\bigr).
    \end{align}
    Taking logarithms gives
    \begin{align}
      \kappa t_{\mathrm{peak},\alpha}
      =
      \log\!\left(\frac{K_{\alpha}}{c_{\mathrm{peak}}}\right)
      - \frac{3}{2}\log t_{\mathrm{peak},\alpha}
      + o(1).
    \end{align}
    Subtracting the empirical and generalization versions yields
    \begin{align}
      \kappa \Delta t_{\text{EU peak}}
      \simeq
      \log\!\left(\frac{K_{\text{gen}}}{K_{\text{emp}}}\right)
      - \frac{3}{2}
      \log\!\left(
        \frac{t_{\mathrm{peak},\mathrm{gen}}}
        {t_{\mathrm{peak},\mathrm{emp}}}
      \right),
    \end{align}
    up to $o(1)$ terms. The logarithmic ratio of peak times is a subleading
    correction coming from the common $t^{-3/2}$ factor, exactly as in the proof
    of \cref{app:thm:grokking-delay}. Thus the leading contribution is
    \begin{align}
      \Delta t_{\text{EU peak}}
      \simeq
      \frac{
        \log(K_{\text{gen}} / K_{\text{emp}})
      }{\kappa},
    \end{align}
    with threshold dependence entering only through the discarded subleading
    terms. Finally, \cref{eq:vi_prefactor_ratio} gives
    \begin{align}
      \log\!\left(\frac{K_{\text{gen}}}{K_{\text{emp}}}\right)
      =
      -2 \log(1 - \sqrt{\lambda}),
    \end{align}
    and
    \begin{align}
      \kappa
      =
      2\left(
        \frac{n}{\sigma^2}(1 - \sqrt{\lambda})^2 + \frac{1}{\tau^2}
      \right).
    \end{align}
    Substituting these expressions gives
    \cref{eq:threshold_eu_peak_delay_lambda}. Thus the threshold-model
    uncertainty peak and the grokking delay are different observables, but near
    $\lambda \simeq 1$ they are governed by the same edge-of-spectrum slowdown.
  \end{proof}
\end{theorem}

\section{Reproducibility}
\label{sec:reproducibility}

We summarize here the full experimental specification needed to reproduce the reported results. Across all experiments, we keep the task family, the transformer architecture, and the evaluation protocol fixed, and vary only the quantities explicitly listed.

\subsection{Task distribution and sequence construction}

Each example is a two-variable modular arithmetic rule
\[
  z=(ax+by) \bmod 29,
\]
where $(a,b)$ identifies the task and $(x,y)$ is the input pair.
The arithmetic base, vocabulary size, and modulus are all $29$.
The input space therefore contains $29^2=841$ distinct pairs $(x,y)$.

For each experimental setting, we first sample $n_{\mathrm{task}}$ in-distribution coefficient pairs and then apply a parallelogram augmentation, which expands the in-distribution family to $4n_{\mathrm{task}}$ related tasks, check  \citet{he2024learning} for details on the augmentation procedure.
Out-of-distribution evaluation uses $256$ additional tasks drawn from the complement of the in-distribution set.

Training and evaluation are performed on sequences of $32$ triplets, corresponding to a token window of length $96$. Only the answer token $z$ in each triplet is supervised, so learning is next-token prediction on answer positions only. The nominal batch size is $1024$ for both training and evaluation. Training uses a virtual in-distribution set multiplier of $1000$ to increase the number of sampled batches per optimization run. Every evaluation reports three splits: in-distribution validation, out-of-distribution training-point evaluation, and out-of-distribution validation-point evaluation.

\begin{table}[t]
  \centering
  \small
  \caption{Task-generation and sequence-construction choices shared across the main experiments.}
  \label{tab:reprod-data}
  \begin{tabular}{p{0.35\linewidth}p{0.56\linewidth}}
    \toprule
    \textbf{Quantity} & \textbf{Value} \\
    \midrule
    Arithmetic rule & $z=(ax+by) \bmod 29$ \\
    Vocabulary size & $29$ \\
    Task dimension & $2$ \\
    Input-pair dimension & $2$ \\
    Input space size & $29^2=841$ \\
    Base in-distribution task count & $n_{\mathrm{task}}$ \\
    Parallelogram augmentation & Enabled in all reported runs \\
    Effective in-distribution task count & $4n_{\mathrm{task}}$ \\
    Out-of-distribution task count & $256$ \\
    Training-point fraction & $\{0.3,0.4,0.5,0.6,0.7,0.8\}$, depending on sweep \\
    Context length & $32$ triplets \\
    Token window & $96$ tokens \\
    Nominal batch size & $1024$ \\
    \bottomrule
  \end{tabular}
\end{table}

\subsection{Shared model architecture}

All methods use the same decoder-only transformer. The backbone has $6$ layers, $4$ attention heads, embedding width $512$, and feed-forward width $2048$. The vocabulary size is $29$ and the maximum token window is $96$. The model uses LayerNorm with $\epsilon=10^{-5}$, full rotary positional encoding, tied input and output embeddings, and no bias terms in attention, MLP, or output projections. All runs use mixed bfloat16 precision and one accelerator per run.

\begin{table}[t]
  \centering
  \small
  \caption{Shared architecture for MAP, IVON, and Laplace evaluations.}
  \label{tab:reprod-model}
  \begin{tabular}{p{0.35\linewidth}p{0.56\linewidth}}
    \toprule
    \textbf{Quantity} & \textbf{Value} \\
    \midrule
    Architecture & Decoder-only transformer \\
    Layers & $6$ \\
    Attention heads & $4$ \\
    Embedding width & $512$ \\
    Feed-forward width & $2048$ \\
    Vocabulary size & $29$ \\
    Maximum token window & $96$ \\
    Normalization & LayerNorm, $\epsilon=10^{-5}$ \\
    Positional encoding & Rotary\\
    Embedding tying & Enabled \\
    Bias terms & Disabled \\
    Numerical precision & Mixed bfloat16 \\
    \bottomrule
  \end{tabular}
\end{table}

\subsection{Training and evaluation protocol}

The default optimization horizon is $100{,}000$ update steps. For the largest base task setting, $n_{\mathrm{task}}=128$, the horizon is doubled to $200{,}000$ steps.
For every validation split we report accuracy, expected calibration error, mean log-likelihood, aleatoric uncertainty, and epistemic uncertainty. Accuracy, calibration, and uncertainty are reported both over all supervised answer positions in the sequence and on the final answer token only.

\subsection{Method-specific settings}

The MAP baseline is trained with AdamW using peak learning rate $1.5\times 10^{-4}$, linear warmup for $1000$ steps from $1\%$ of the peak rate, and cosine decay to $10\%$ of the peak rate. The momentum parameters are $(0.9,0.98)$, the numerical stability constant is $10^{-8}$, and the gradient norm is clipped at $1.0$.

IVON uses learning rate $0.5$, $2000$ warmup steps, the same $1\%\rightarrow100\%\rightarrow10\%$ cosine schedule envelope, $\beta_1=0.9$, $\beta_2=0.99999$, Hessian initialization $1$, clip radius $10^{-3}$, one Monte Carlo weight sample during training.
The effective sample size is set automatically to the number of in-distribution training point pairs multiplied by the effective in-distribution task count after augmentation.

The Laplace model is fit post hoc on the final MAP solution. The approximation is restricted to the last layer, uses a Kronecker-factored Hessian, and sets the prior precision equal to the MAP weight decay (no marginal likelihood optimization). The Laplace fit is always performed on clean in-distribution training data, with full context sequences of $32$ triplets.

\begin{table}[t]
  \centering
  \small
  \caption{Optimization settings for the two trained model families.}%
  \label{tab:reprod-optim}%
  \begin{tabular}{p{0.27\linewidth}p{0.3\linewidth}p{0.3\linewidth}}
    \toprule
    \textbf{Field} & \textbf{MAP (AdamW)} & \textbf{IVON} \\
    \midrule
    Peak learning rate & $1.5\times 10^{-4}$ & $0.5$ \\
    Warmup steps & $1000$ & $2000$ \\
    Initial learning-rate scale & $0.01$ & $0.01$ \\
    Final learning-rate scale & $0.1$ & $0.1$ \\
    Schedule & Linear warmup + cosine decay & Linear warmup + cosine decay \\
    Weight decay & $1$ &$10^{-6}$  \\
    Momentum parameters & $(0.9,0.98)$ & $(0.9,0.99999)$ \\
    Stability / curvature init & $\epsilon=10^{-8}$ & Hessian init $=1$ \\
    Gradient control & Norm clip at $1.0$ & Clip radius $10^{-3}$ \\
    Training samples per step & Deterministic model & $1$ Monte Carlo sample \\
    Evaluation samples & Deterministic model & $16$/$64$ Monte Carlo samples\\
    \bottomrule
  \end{tabular}
\end{table}

\begin{table}[t]
  \centering
  \small
  \caption{Laplace-specific settings.}
  \label{tab:reprod-laplace}
  \begin{tabular}{p{0.35\linewidth}p{0.56\linewidth}}
    \toprule
    \textbf{Field} & \textbf{Laplace approximation} \\
    \midrule
    Base solution & Final MAP model \\
    Bayesian scope & Last layer only \\
    Curvature structure & Kronecker-factored Hessian \\
    Prior precision & Equal to MAP weight decay \\
    Posterior samples & $16$/$64$ \\
    Fit data & Clean in-distribution training data \\
    Role & Post-hoc predictive uncertainty estimation \\
    \bottomrule
  \end{tabular}
\end{table}

\subsection{Experiment families}

The main MAP study sweeps base task count, training-point fraction, and initialization seed. Specifically, it uses $n_{\mathrm{task}}\in\{8,16,32,64,128\}$, $p_{\mathrm{train}}\in\{0.3,0.4,0.5,0.6,0.7,0.8\}$, and 4 seeds. The full Laplace study is aligned exactly with this grid, using the corresponding final MAP solutions.

The primary IVON study uses the same sweep axes over base task count and training-point fraction, with 4 seeds. Unless otherwise stated, these runs use the default IVON weight decay $10^{-5}$, automatic effective sample size resolution, and $100{,}000$ optimization steps, with the same $200{,}000$-step extension for $n_{\mathrm{task}}=128$.

Inference-time context-length sweeps evaluate context sizes $1,2,\ldots,32$ while keeping the trained model fixed. These sweeps use $n_{\mathrm{task}}=64$ and $p_{\mathrm{train}}=0.8$.

Inference-time context-noise sweeps also use $n_{\mathrm{task}}=64$ and $p_{\mathrm{train}}=0.8$, with fixed context length $32$. Noise probabilities range over $\{0.0,0.1,\ldots,1.0\}$. Corruption is applied only to eligible context answer tokens; the query triplet and evaluation targets remain clean.

\subsection{Hardware allocation}

All runs use one accelerator per experiment. The main training runs use NVIDIA H200/H100/A100 accelerators while  lightweight validation use small H200 MIG partitions.
The total computational budget used for the experimental campaign is approximately $12'000$ accelerator-hours, splitted roughly equally between development and final runs.

\section{LLM use disclosure}

In accordance with the COLM policy on large language model usage, we disclose that LLM-based tools were used for software engineering tasks and editorial assistance, specifically for code implementation support and code review (beyond simple auto-completion), and rewriting or polishing text (beyond simple grammar checking).
They were not used to originate research ideas or formulate scientific claims. The authors take full responsibility for the design of the study, the experiments, the analysis, and the final manuscript.

\end{document}